\documentclass[12pt]{article}
\usepackage{natbib}
\usepackage{times}
\usepackage{fullpage}
\usepackage{authblk}
\usepackage[utf8]{inputenc}

\usepackage{amsmath, amsfonts, amsthm, amssymb, graphicx}
\usepackage{mathtools,verbatim}
\usepackage{enumitem}
\usepackage{natbib}
\usepackage{algorithm, algpseudocode}
\usepackage{hyperref}
\hypersetup{colorlinks,citecolor=blue}
\usepackage{fancyhdr}
\usepackage{cleveref}
\numberwithin{equation}{section}

\newcommand{\eps}{\epsilon}

\renewcommand{\hbar}{\tau}

\DeclarePairedDelimiter\ceil{\lceil}{\rceil}

{
 \theoremstyle{plain}
      \newtheorem{asm}{Assumption}
}
\newtheorem{nono-theorem}{Theorem}[]
\theoremstyle{plain}
\newtheorem{theorem}{Theorem}[section]

\newtheorem{lemma}[theorem]{Lemma}
\newtheorem{corollary}[theorem]{Corollary}

\theoremstyle{definition}

\newtheorem{remark}[theorem]{Remark}

\newcommand{\Exp}{\mathbb{E}}

\newcommand{\tr}{\mathrm{tr}}

\newcommand{\R}{\mathbb{R}}

\DeclareMathOperator*{\argmin}{arg\,min}

\newcommand{\calF}{\mathcal{F}}

\newcommand{\calX}{\mathcal{X}}

\newcommand{\Mtwo}{V_{\max}}

\graphicspath{{figures/}}
\usepackage{booktabs}

\newcommand{\algname}{LC\textsuperscript{3}}
\newcommand{\algfullname}{Lower Confidence-based Continuous Control}

\title{Information Theoretic Regret Bounds \\
for Online Nonlinear Control}

\author[1,2]{Sham Kakade}
\author[2]{Akshay Krishnamurthy}
\author[1]{\\Kendall Lowrey}
\author[1]{Motoya Ohnishi}
\author[2]{Wen Sun}
\affil[1]{University of Washington}
\affil[2]{Microsoft Research NYC}

\begin{document}
\maketitle

\begin{abstract}
This work studies the problem of sequential control in an unknown,
\emph{nonlinear} dynamical system, where we model the underlying system
dynamics as an unknown function in a known Reproducing Kernel Hilbert
Space. This framework yields a general setting that permits discrete and
continuous control inputs as well as non-smooth, non-differentiable dynamics.
Our main result, the \algfullname{} (\algname{}) algorithm, enjoys a near-optimal
$O(\sqrt{T})$ regret bound against
the optimal controller in episodic settings, where $T$
is the number of episodes.
The bound has \emph{no} explicit dependence on dimension of the system
dynamics, which could be infinite, but instead only depends on
information theoretic quantities.  
We empirically show its application to a number of \emph{nonlinear} control tasks and demonstrate the benefit of exploration for learning model dynamics.
\end{abstract}

\section{Introduction}

The control of uncertain dynamical systems is one of the central challenges in
Reinforcement Learning (RL) and continuous control, and  recent years has seen a
number of successes in demanding sequential decision making tasks ranging from
robotic hand manipulation~\citep{todorov2012mujoco,Hertzmann,Vikash16,Tobin17,lowrey2018plan,openai2019solving} to 
game playing~\citep{AlphaGo,bellemare2016pseudocounts,pathak2017curiosity,burda2018exploration}. The predominant
approaches here are either based on reinforcement learning or
continuous control (or a mix of techniques from both domains).

%A central challenge for both RL and controls is handling the 
%uncertainty of the underlying model dynamics.
With regards to provably correct methods which handle both the
learning and approximation in
unknown, complex environments, the body of results in the
reinforcement learning literature~\citep{russo2013eluder,jiang2017contextual,sun2018model,agarwal2019optimality} is
far more mature than in the continuous controls literature.  
In fact, only relatively
recently has there been provably correct methods (and sharp bounds) for the learning
and control of the Linear Quadratic
Regulator (LQR) model~\citep{mania2019certainty,simchowitz2020naive,abbasi2011regret},
arguably the most basic model due to having globally \emph{linear} dynamics. 

While Markov Decision Processes provide a very
general framework after incorporating continuous states and
actions into the model, there are a variety of reasons to directly
consider learning in continuous
control settings: even the simple LQR model provides a powerful framework
when used for \emph{locally} linear planning~\citep{ahn2007iterative,Todorov04,tedrake2009lqr,perez2012lqr}.  More
generally, continuous control problems often have
continuity properties with respect to the underlying
``disturbance'' (often modeled as statistical additive noise), which can be
exploited for fast path planning algorithms~\citep{jacobson1970differential,williams2017model};
analogous continuity properties are often not leveraged in designing
provably correct RL models
(though there are a few exceptions, e.g. \citep{DBLP:conf/icml/KakadeKL03}). While LQRs are a natural model for continuous
control, they are prohibitive for a variety of reasons: LQRs rarely
provide good \emph{global} models of the system dynamics, and, furthermore, naive random
search suffices for sample efficient learning of
LQRs~\citep{mania2019certainty,simchowitz2020naive} --- a strategy which is unlikely to be effective
for the learning and control of more complex nonlinear dynamical
systems where one would expect strategic exploration to be required for
sample efficient learning (just as in RL, e.g. see ~\cite{kearns2002near,KakadeThesis}).

This is the motivation for this line of work, where we focus directly on the
sample efficient learning and control of an \emph{unknown}, nonlinear dynamical
system, under the assumption that the mean dynamics live within some
known Reproducing Kernel Hilbert Space. 
%As we shall this poses a number of new challenges in contrast to the control of LQRs.

\iffalse
While there is an increasingly large body of work 
providing provably correct reinforcement learning methods which handle large state and
action spaces \cite{BellmanRank,EluderDim,Witness,PolicyGrad}, the
situation is far less mature with regards to provably efficient methods
for unknown continuous control problems.  In fact, only relatively
recently have there been provably correct methods for the learning
and control of \emph{linear} dynamical systems, in the Linear Quadratic
Regulator model\cite{BensWork,etc}.  While LQRs are quite powerful
models when used as a tool for locally linear planning~\cite{EmosWork}, they
are quite limiting as global models.  Furthermore, while 
naive random search suffices for the sample efficient learning of LQRs\cite{Mania,Max}, it is
unreasonable to expect naive random search to be effective in the learning
and control of more challenging nonlinear dynamical systems.

The focus of this work is on control of an \emph{unknown},
nonlinear dynamical systems. We will work in a setting where we 
assume the mean dynamics live within some known Reproducing Kernel
Hilbert Space, and our focus is on sample efficient control. As we
shall this poses a number of new challenges in contrast to 
the control of LQRs. 
%In particular, as has been show recently in exploration
\fi

\paragraph{The Online Nonlinear Control Problem.} This work studies
the following nonlinear control problem, where the nonlinear system
dynamics are described, for $h\in \{0, 1, \ldots
H-1\}$, by
\begin{align*}%\label{eq:dynamicalsystem}
  x_{h+1}=f(x_h,u_h) + \epsilon, \textrm{ where } \epsilon \sim
  \mathcal{N}(0,\sigma^2 I) 
\end{align*}
where the state $x_h\in\R^{d_\mathcal{X}}$; the control $u_h \in\mathcal{U}$
where $\mathcal{U}$ may be an arbitrary set (not necessarily even a
vector space); %$I$ is the $d \times d$ identity matrix; 
 $f:\mathcal{X}\times \mathcal{U}\rightarrow
\mathcal{X}$ is assumed to live within some known Reproducing Kernel
Hilbert Space; the additive noise is assumed to be independent
across timesteps.

%To our knowledge, the earliest known consideration of
%this model was in PILCO~\citep{deisenroth2011pilco}, which modeled the
%dynamics as a Gaussian process.

Specifically, the model considered in this work was recently
introduced in~\cite{horia_sys_id}, which we refer to as the \emph{kernelized
nonlinear regulator} (KNR) for the infinite dimensional extension. The KNR model assumes that $f$ lives
in the RKHS of a known kernel $K$. Equivalently, the primal version of
this assumption is that:
\begin{equation*} %\label{eq:KNR_ddf}
f(x,u) = W^\star \phi(x,u)
\end{equation*}
for some known function
$\phi: \mathcal{X}\times\mathcal{U} \rightarrow \mathcal{H}$ where
$\mathcal{H}$ is a Hilbert space (either finite or countably infinite dimensional)
and where $W^\star$ is a linear mapping. Given an immediate cost
function $c:\mathcal{X} \times \mathcal{U}\rightarrow \R^{+}$ (where
$\R^{+}$ is the non-negative real numbers), the KNR problem can be
described by the following optimization problem:
\[
\min_{\pi \in \Pi} \ J^\pi(x_0;c) \,
\textrm{ where }
J^\pi(x_0;c) =  \Exp \left[ \sum_{h=0}^{H-1} c(x_h, u_h)  \Big|  \pi,  x_0 \right] 
\]
where $x_0$ is a given
starting state; $\Pi$ is some set of feasible controllers;
 and where a controller (or a policy) is a mapping $\pi: \mathcal{X}\times \{0,\ldots H-1\} \rightarrow
\mathcal{U}$.   We denote the best-in-class cumulative cost as
$J^\star(x_0;c)=\min_{\pi\in\Pi}J^\pi(x_0;c)$.  
Given any model parameterization $W$, we denote $J^\pi(x_0;c, W)$ as the expected total cost of $\pi$ under the dynamics $W\phi(x,u)+\epsilon$.

It is worthwhile to note that this KNR model is rather general in the
following sense: the
space of control inputs $\mathcal{U}$ may be either discrete or
continuous; and the dynamics $f$ need not be a smooth or differentiable
function in any of its inputs. A more general version of this problem,
which we leave for future work, 
would be where
the systems dynamics are of the form $
x_{h+1}=f_h(x_h,u_h,\epsilon_h)$, in
contrast to our setting where the disturbance is due to
additive Gaussian noise.

We consider an online version of this KNR problem: the objective is to sequentially optimize a 
sequence of cost functions where the nonlinear dynamics $f$ are
not known in advance. We assume that the learner
knows the underlying Reproducing Kernel Hilbert
Space. In each episode $t$, we observe an instantaneous cost function $c^t$; we
choose a policy $\pi^t$; we execute $\pi^t$ and observe a sampled trajectory $x_0,u_0, \ldots,
x_{H-1},u_{H-1}$; we incur the cumulative cost under $c^t$. Our goal is to minimize the sum
of our costs over $T$ episodes. In particular, we desire to execute a policy
that is nearly optimal for every episode.

A natural performance metric in this context is our cumulative regret, the
increase in cost due to not knowing the nonlinear dynamics
beforehand, defined as:
\[
%\Exp\left[\sum_{t=0} ^{T-1} \sum_{h=0}^{H-1} c^t(x^t_h, u^t_h) - \sum_{t=0} ^{T-1} \min_{\pi\in\Pi} J^\pi(x_0;c^t)\right]
\textsc{Regret}_T = 
\sum_{t=0} ^{T-1} \sum_{h=0}^{H-1} c^t(x^t_h, u^t_h) - \sum_{t=0} ^{T-1} \min_{\pi\in\Pi} J^\pi(x_0;c^t)
\]
where $\{x^t_h\}$ is the observed states and $\{u^t_h \}$ is the
observed sequence of controls.
A desirable asymptotic property of an algorithm is to be no-regret, i.e.
the time averaged version of the regret goes to $0$ as $T$ tends
to infinity.

\paragraph{Our Contributions.} 

The first set of provable results in this setting, for the finite dimensional
case and for the problem of system identification, was provided by \cite{horia_sys_id}.
Our work focuses on regret, and we provide the \algfullname{} (\algname{}) algorithm, which enjoys a $O(\sqrt{T})$
regret bound. We provide an informal version of our main
result, specialized to the case where the dimension
of the RKHS is finite and the costs are bounded.

\begin{theorem}[Informal statement; finite dimensional case with
  bounded features] 
Consider the special case where: $c^t(x,u) \in [0,1]$;
$d_\phi$ is  finite (with $d_{\mathcal{X}} +d_\phi\geq H$);  and $\phi$ is uniformly
bounded, with $\|\phi(x,u)\|_2 \leq B$;
The \algname{} algorithm enjoys the following expected regret bound:
%$\widetilde{O}(\sqrt{d_\phi (d_\mathcal{X}+d_\phi) H^3 T})$
\begin{align*}
\Exp_{\mathrm{\algname{}}}\left[\textsc{Regret}_T\right] \leq  
\widetilde{O}\left(\sqrt{ d_{\phi}\big(d_{\mathcal{X}} + d_{\phi} \big)  H^3T  } \cdot 
\log\left(1 + \frac{B^2\|W^\star\|_2^2}{\sigma^2} \right)\right),
\end{align*}
where $\widetilde{O}(\cdot)$ notation drops logarithmic factors in
$T$ and $H$.
\end{theorem}

There are a number of notable further contributions with regards to our
work:
\begin{itemize}
\item (\emph{Dimension and Horizon Dependencies})  Our general regret bound 
has \emph{no} explicit dependence on dimension of the system
dynamics (the RKHS dimension), which could be infinite, but instead only depends on
information theoretic quantities; our
horizon dependence is $H^3$, which we conjecture is near optimal. It
is also worthwhile noting that our regret bound is only logarithmic in
$\|W^\star\|_2$ and $\sigma^2$.
\iffalse
While we do not provide lower bounds which leave to future work, we conjecture our rate is
  tight in various quantities. For $H=1$, it matches the lower bounds
  in~\citep{dani2008stochastic}. Furthermore, our $H^2$  horizon dependence
  matches the lower bounds in~\cite{azar2017minimax,dann2015sample} for MDPs; technically,
  obtaining our leading order dependence required a number of analysis
  tools.
\fi
\item (\emph{Localized rates}) In online learning, it is
  desirable to obtain improved rates if the loss of the ``best
  expert'' is small, e.g. in our case, if $J^\star(x_0;c^t)$
  is small. Under a bounded coefficient of variation condition (which
  holds for LQRs and may hold more generally), we provide an improved
  regret bound whose leading term regret depends linearly on
  $J^\star$.
\item (\emph{Moment bounds and LQRs}) Our regret bound does not require
bounded costs, but instead only depends
on second moment bounds of the realized cumulative
cost, thus making them applicable to  LQRs, as a special case.
%\item (LQRs) Our regret bound can be specialized to the case of
%  LQRs. While our setting is more general, when specialized to the
%  case of an LQR our dimension dependence is near optimal.
\item (\emph{Empirical evaluation:}) Coupled with the right features (e.g., kernels), our method, arguably one of the simplest, provides competitive results in common continuous control benchmarks, exploration tasks, and complex control problems such as dexterous manipulation.
%a complicated dexterous manipulation task. 
%simple kernels (or Random Fourier features), our method provides competitive results 
%In a sense, coupled with simple kernels (or random
%  Fourier features) as suggested in, one could argue that this method is one
%  of simplest methods to provide a competitive baselines, analogous to
%  the work in ~\cite{NIPS2017_7233}. \sk{please fix this. what should say here?}
\end{itemize}

While our techniques utilize methods developed for the analysis of linear
bandits \citep{dani2008stochastic,abbasi2011improved} and Gaussian
process bandits \citep{srinivas2009gaussian},  there are a number of
new technical challenges to be addressed with regards to the 
multi-step extension to Reinforcement Learning.  In
particular, some nuances for the more interested reader: 
we develop a stopping time martingale to handle the unbounded nature
of the (realized) cumulative costs;
we develop a novel way to handle Gaussian smoothing through the
chi-squared distance function between two distributions; our main
technical lemma is a ``self-bounding'' regret bound that relates the
instantaneous regret on any given episode to the second moment of the
stochastic process.

\paragraph{Notation.}  We let $\|x\|_2$, $\|M\|_2$, and $\|M\|_F$ refer to the
Euclidean norm, the spectral norm, and the Frobenius norm,
respectively,  of a vector $x$ and a matrix $M$. 
%We slightly abuse notation and let $[H-1]$ refer to the set $\{0,\ldots H-1\}$.

%!TEX root = main.tex

\section{Related Work}\label{section:rel}

The first set of provable results with regards to this nonlinear
control model was provided by \cite{horia_sys_id}, who studied the problem of system identification in a finite dimensional setting
(we discuss these results later).  While not appearing with this name,
a Gaussian process version of this model was originally considered
by~\citet{deisenroth2011pilco}, without sample-efficiency guarantees.  

More generally, most model-based RL/controls algorithms do not
explicitly address the exploration challenge nor do they have
guarantees on the performance of the learned policy
\citep{deisenroth2011pilco,levine2014learning,chua2018deep,kurutach2018model,nagabandi2018neural,luo2018algorithmic,ross2012agnostic}.
Departing from these works, we focus on provable sample efficient
regret bounds and
strategic exploration in model-based learning in the kernelized
nonlinear regulator.

Among provably efficient model-based algorithms, works
from~\cite{sun2018model,osband2014model,ayoub2020model,lu2019information} are the
most related to our work.
%  While these algorithms leverage the
%principle of optimism in the face of uncertainty for strategic
%exploration in structured MDPs with large state space. 
While these
works are applicable to certain linear structures, their techniques do
not lead to the results herein: even for the special case of LQRs, they do
not address the unbounded nature of the costs, where there is more
specialized
analysis~\citep{mania2019certainty,cohen2019learning,simchowitz2020naive};
these results do not provide techniques for sharp leading order
dependencies like in our regret results (and the former three do not handle the infinite dimensional case); they also do not provide
techniques for localized regret bounds, like those herein which depend
on $J^\star$.  
A few more specific differences: the model complexity measure Witness Rank
from~\cite{sun2018model} does contain the kernelized nonlinear
regulator if the costs were bounded and the dimensions were finite;
\cite{osband2014model} considers a setting where the model class has small
Eluder dimension, which does not apply to the infinite-dimensional
settings that we consider here; \cite{lu2019information} presents a
general information theoretic framework providing results for tabular
and factor MDPs.  There are numerous technical challenges addressed in
this work which may be helpful for further analysis of models in
continuous control problems.

Another family of related work provides regret analyses of online LQR
problems.  There are a host of settings considered: unknown stochastic
dynamics~\citep{abbasi2011regret,dean2018regret,mania2019certainty,cohen2019learning,simchowitz2020naive};
adversarial noise (or adversarial noise with unknown mean
dynamics)~\citep{agarwal2019online,hazan2019nonstochastic}; changing
cost functions with known dynamics~\citep{cohen2018online,agarwal2019logarithmic}. For the
case of unknown (stochastic) dynamics, our online KNR  problem is more
general than these works, due to a more general underlying model; one
distinction is that many of these prior works on LQRs consider the regret
on a single trajectory, under stronger stability and mixing
assumptions of the process. This is an interesting direction for future work (see
Section~\ref{section:discussion}).

On the system identification side, \cite{horia_sys_id} provides the
first sample complexity analysis for finite dimensional KNRs under assumptions of
reachability and bounded features. The work
of~\cite{horia_sys_id} is an important departure from the
aforementioned model-based theoretical
results~\citep{sun2018model,osband2014model,ayoub2020model,lu2019information};
the potentially explosive nature of the system dynamics makes system
ID challenging, and ~\cite{horia_sys_id} directly addresses this
through various continuity assumptions on the dynamics. One notable
aspect of our work is that it permits both an unbounded state and
unbounded features. The KNR also has been used in practice for system
identification \citep{ng2006autonomous,abbeel2005exploration}.

%!TEX root = main.tex
\section{Main Results}

\subsection{The \algfullname{} algorithm}

\begin{algorithm}[!t]
	\begin{algorithmic}[1]	
		\Require Policy class $\Pi$; regularizer $\lambda$;
                confidence parameter $C_1$ (see Equation~\ref{eq:beta}).
                \State Initialize  $\textsc{Ball}^{0}$ to be any set containing $W^\star$.
		\For{$t = 0 \dots T$}
			\State $\pi^t = \argmin_{\pi\in\Pi}\min_{W\in\textsc{Ball}^t} J^{\pi}(x_0; c^t, W)$ \label{alg:ucb_step}
			\State Execute $\pi^t$ to sample a trajectory $\tau^t := \{x^t_h,u^t_h, c^t_h, x^t_{h+1}\}_{h=0}^{H-1}$
			\State Update $\textsc{Ball}^{t+1}$ (as specified in Equation~\ref{eq:ball}).
		\EndFor
         \end{algorithmic}
  \caption{\algfullname{} (\algname{})}
\label{alg:lc3}
\end{algorithm}

\algname{} is based on ``optimism in the face of uncertainty,'' which is
described in Algorithm~\ref{alg:lc3}. At episode $t$, we use all previous experience to define an uncertainty region (an ellipse).  The center
of this region,  $\overline{W}^t$, is the solution of the following regularized least squares problem:
\begin{equation}\label{eq:regression}
\overline{W}^t = \arg\min_{W} \sum_{\tau = 0}^{t-1}\sum_{h = 0}^{H-1}
\| W \phi(x^\tau_h,u^\tau_h) - x^\tau_{h+1} \|_2^2 + \lambda\|W\|_F^2,
\end{equation}
where $\lambda$ is a parameter, 
and the shape of the region is defined through the feature covariance: 
\[
\Sigma^t = \lambda I + \sum_{\tau = 0}^{t-1}\sum_{h = 0}^{H-1} \phi(x^\tau_h,u^\tau_h) (\phi(x^\tau_h,u^\tau_h))^\top,
\, \textrm{with } \, \, \, \Sigma^0 = \lambda I.
\] 
The uncertainty region, or confidence ball, is defined as:
\begin{align}\label{eq:ball}
\textsc{Ball}^{t} = \left\{ W \Big\vert \left\| \left(W - \overline{W}^t \right) \left(\Sigma^t\right)^{1/2} \right\|^2_2  \leq \beta^t \right\},
\end{align}
where recall that $\|M\|_2$ denotes the spectral norm of a matrix $M$
and where
\begin{align}\label{eq:beta}
\beta^t := C_1\bigg({\lambda}\sigma^2 + \sigma^2\Big( d_{\mathcal{X}} 
+ \log\left(t \det(\Sigma^t)/\det(\Sigma^0)  \right)
  \Big) \bigg),
\end{align}
with $C_1$ being a parameter of the algorithm. 

At episode $t$, the \algname{} algorithm will choose an optimistic policy in
Line~\ref{alg:ucb_step} of Algorithm~\ref{alg:lc3}.
Solving this optimistic planning problem in general is NP-hard
\citep{dani2008stochastic}. Given this computational hardness, we focus
on the statistical complexity and explicitly assume access to
the following computational oracle:
\begin{asm} [Black-box computation oracle]
\label{asm:compute}
We assume access to an oracle that implements Line~\ref{alg:ucb_step} of Algorithm~\ref{alg:lc3}.
\end{asm} 
We leave to future work the question of finding reasonable approximation
algorithms, though we observe that a number of effective heuristics
may be available through gradient based methods such as DDP~\citep{jacobson1970differential},
iLQG~\citep{Todorov04} and CIO~\citep{mordatch2012discovery}, or sampling based methods, such as MPPI~\citep{williams2017model} and DMD-MPC~\citep{DBLP:journals/corr/abs-1902-08967}. 
In particular, these planning algorithms are natural to use in conjunction with Thompson Sampling 
\citep{thompson1933likelihood,osband2014model}, i.e. we sample $W^t$
from $\mathcal{N}(\overline{W}^t, (\Sigma^t)^{-1})$ and then compute
and execute the
corresponding optimal policy $\pi^t = \argmin_{\pi\in\Pi} J^{\pi}(x_0;
c^t, W^t)$ using a planning oracle. While we focus on the
frequentist regret bounds, we conjecture that a Bayesian regret bound for
the Thompson sampling algorithm, should be achievable using the
techniques we develop herein, along with now standard techniques for 
analyzing the Bayesian regret of Thompson sampling
(e.g. see ~\cite{russo2014learning}).

%\subsection{Posterior Sampling}

\subsection{Information Theoretic Regret Bounds}

In this section, we analyze the regret of Algorithm~\ref{alg:lc3}.
Following~\cite{srinivas2009gaussian},  let us define the (expected)
Maximum Information Gain as:
\begin{eqnarray*}
\gamma_T(\lambda) & := & 
\max_{\mathcal{A}}  \Exp_\mathcal{A} \left[\log \Big(\textrm{det} \left(\Sigma^{T}\right)
/\textrm{det} \left(\Sigma^{0})\right) \Big) \right]\\
&= & 
\max_{\mathcal{A}}  \Exp_\mathcal{A} \left[ \log \textrm{det} \left( I + 
\frac{1}{\lambda}\sum_{t=0}^{T-1} \sum_{h=0}^{H-1} \phi(x^t_h,u^\tau_h) (\phi(x^t_h,u^\tau_h))^\top \right)
%\frac{1}{\lambda}\sum_{t=0}^{T-1} \sum_{h=0}^{H-1} \phi^t_h (\phi^t_h)^\top \right)
\right],
\end{eqnarray*} 
\iffalse
\begin{eqnarray*}
\gamma_T(\lambda) & := & 
\max_{\{\pi^i\}_{i=0}^{T-1}}  \ \log \left(\textrm{det} \left(\Exp \left[\Sigma^{T-1}\right]\right)
/\textrm{det} (\Sigma^{0})) \right)\\
&= & 
\max_{\{\pi^i\}_{i=0}^{T-1}} \ \log \textrm{det} \left( I + 
\frac{1}{\lambda}\Exp \left[\sum_{t=0}^{T-1} \sum_{h=0}^{H-1} \phi^t_h (\phi^t_h)^\top \right]\right),
\end{eqnarray*} 
\fi
where the max is over algorithms $\mathcal{A}$, where an algorithm is
a mapping from the history before episode $t$ to the next policy
$\pi_t\in\Pi$. %The  last equality above is proved in Appendix~\ref{section:technical}.
%Lemma~\ref{lemma:info_gain_eq}.  

\begin{remark}\label{remark:info_gain}
(Finite dimensional RKHS) For
$\phi \in\mathbb{R}^{d_\phi}$, with $\|\phi(x,u)\|\leq B\in\mathbb{R}^+$ for all $(x,u)$, then
$\gamma_T(\lambda)$ will be $O(d_\phi\log(1+ TH B^2/\lambda)$  (see Lemma~\ref{lem:finite_info_gain}). Furthermore, it may be the
case that $\gamma_T(\lambda ) \ll d_\phi$ if the eigenspectrum of the covariance matrices of the
policies tend to concentrate in a lower dimensional subspace.  See
\cite{srinivas2009gaussian} for details and for how $\gamma_T(\lambda )$ scales for a
number of popular kernels.
\end{remark}

\subsubsection*{The General Case, with Bounded Moments}

\begin{asm}\label{asm:moment_bound} (Bounded second moments at $x_0$) Assume that $c^t$ is a non-negative function for all $t$ and  
that  the realized cumulative cost, when starting from $x_0$, has uniformly bounded second
moments, over all policies and cost functions $c^t$. Precisely, suppose for every $c^t$, 
\[
\sup_{\pi\in\Pi}  \, \Exp\left[\left(\sum_{h=0}^{H-1} c^t(x_h, u_h)\right)^2
 \ \bigg\vert \  x_0, \pi \right]
\leq \Mtwo.
\]
\end{asm}
This assumption is substantially weaker than the standard bounded cost
assumption used in prior model-based works
\citep{sun2018model,lu2019information}; furthermore, the assumption only depends on the
starting $x_0$ as opposed to a uniform bound over the state space.

\begin{theorem}[\algname{} regret bound] \label{thm:main1}
Suppose Assumptions~\ref{asm:compute} and~\ref{asm:moment_bound} hold. 
Set $\lambda =\frac{ \sigma^2}{\|W^\star\|_2^2}$  and define %the ``effective dimension'' as
\[
\widetilde d_T^{\ 2} := \gamma_T(\lambda) \cdot 
\big( \gamma_T(\lambda) + d_{\mathcal{X}} + \log(T) +H\big).
\]
There exist constants $C_1,C_2\leq 20$ such that if \algname{} (Alg.~\ref{alg:lc3}) is run with
input parameters $\lambda $ and $C_1$ (in Equation \ref{eq:beta}), then following regret bound holds for all $T$,
\begin{align*}
\Exp_{\mathrm{\algname{}}}\left[\textsc{Regret}_T\right]&  \leq  
C_2 \ \widetilde d_T\sqrt{ \Mtwo H T} .
\end{align*}
\end{theorem} 

While the above regret bound is applicable to the infinite dimensional RKHS
setting and does not require uniformly bounded features $\phi$, it is informative to
specialize the regret bound to the finite dimensional case with bounded features.

%For special case where $\phi \in\mathbb{R}^{d_{\phi}}$ and $\|\phi\|_2 \leq B$, we have the following corollary. 

\begin{corollary}[\algname{} Regret for finite dimensional, bounded features]
Suppose that Assumptions~\ref{asm:compute} and~\ref{asm:moment_bound}
hold; $d_\phi$ is  finite;  and that $\phi$ is uniformly
bounded, with $\|\phi(x,u)\|_2 \leq B$.
Under the same parameter choices as in Theorem~\ref{thm:main1}, we
have, for all $T$,
\begin{align*}
\Exp_{\mathrm{\algname{}}}\left[\textsc{Regret}_T\right]&  \leq  
C_2 \sqrt{ d_{\phi}\Big(d_{\phi} + d_{\mathcal{X}} + \log(T)+H\Big)  V_{\max}HT  } \cdot 
\log\left(1 + \frac{B^2\|W^\star\|_2^2}{\sigma^2} \frac{TH}{d} \right).
\end{align*}
\end{corollary}

The above immediately follows from a bound on the finite dimensional
information gain (see Lemma~\ref{lem:finite_info_gain}). 

A few remarks are in order:

\begin{remark}[Logarithmic parameter dependencies] \label{remark:dimension}
It is worthwhile noting that our regret bound has only logarithmic
dependencies $\|W^\star\|_2$ and $\sigma^2$.  Furthermore, in the case
of finite dimensional and bounded features, the bound is also only logarithmic in the bound $B$.
\end{remark}

\begin{remark}[Dimension and horizon dependencies] \label{remark:dimension}
For the special case with bounded $c^t(x,u) \in [0,1]$, 
bounded $\phi\in\mathbb{R}^{d_\phi}$, and $d_\phi +d_\mathcal{X}
\geq H$, \algname{} has a regret bound of
$\widetilde{O}(\sqrt{d_\phi (d_\phi +d_\mathcal{X}) H^3 T})$. 
Our dimension dependence matches the lower bounds
  in~\citep{dani2008stochastic} for linear bandits (where $H=1$ and
  $d_{\mathcal{X}}=1$). Furthermore, for fixed dimension, one
  might expect an $\Omega(\sqrt{H^2T})$  lower bound based on results for tabular MDPs (see~\cite{azar2017minimax,dann2015sample}).
  %If this is the case, then our horizon dependence, in
  %addition to being polynomial, would be nearly optimal. 
  Obtaining sharp lower bounds is an important direction for future work.
\end{remark}

\begin{remark}[Linear Quadratic Regulators (LQR) as a special case]
Our model generalizes the Linear Quadratic Regulator (LQR). Specifically,
we can set $\phi(x,u) = [x^{\top}, u^{\top}]^{\top}$, $c(x,u) =
x^{\top} Q x + u^{\top}R u$ with $Q$ and $R$ being some PSD matrix.
We can consider a policy class to be a (subset) of all linear controls, i.e., $\Pi
= \{\pi: u = Kx, K\in\mathcal{K}\subset\mathbb{R}^{d_{u}\times d_{\calX}}\}$. 

Consider the case where $d_{\calX} = d_{u} =
d$ (with $d>H$) and the policy class consists of controllable policies (e.g. see~\cite{cohen2019learning}).
Here, our regret scales as
$\widetilde{O}\left(\sqrt{ H^3 d^4 T}\right)$ ( since $\Mtwo =
O(Hd^2)$, e.g. see~\cite{simchowitz2020naive}).
While our rate is a factor of $\sqrt{d}$ larger than
the minimax regret bound for an LQR~\citep{simchowitz2020naive}, which is 
$\Omega(\sqrt{d^3 T})$, our results also apply to non-linear
settings, as opposed to the globally linear LQR setting.  
\end{remark}

\subsubsection*{The Stabilizing Case, with Bounded Coefficient of Variation}

In many cases of practical interest, the optimal cost
$J^\star(x_0; c)$ may be substantially less than the cost of other
controllers, i.e. $J^\star(x_0; c) \ll
\max_{\pi\in\Pi}J^{\pi}(x_0; c)<\sqrt{\Mtwo}$. In such cases, one might
hope for an improved regret bound for sufficiently large $T$. We show that this is the case provided our policy
class satisfies a certain bounded coefficient of variation condition,
which holds for LQRs. Recall the \emph{coefficient of variation} of a
random variable is defined as
the ratio of the standard deviation to mean.

\begin{asm} [Bounded coefficient of variation at $x_0$]
\label{asm:variance_cond}
Assume that the realized cumulative
cost, when starting from $x_0$, has a uniformly bounded coefficient of
variation.
%\footnote{Recall the coefficient of variation is defined as
%  the ratio of the standard deviation to mean.} 
Specifically,
assume there exists an $\alpha \in \mathbb{R}^+$, such that
for every $c^t$ and all $\pi\in\Pi$, 
\begin{align*}
\Exp\left[\left(\sum_{h=0}^{H-1} c^t(x_h,  u_h)\right)^2  \ \bigg\vert \  x_0, \pi \right]
\leq 
\Big(\alpha \ J^{\pi}(x_0; c^t)\Big)^2.
\end{align*}
\end{asm}

\begin{remark}[$\alpha$ for LQRs] It is straightforward to verify
  that Assumption~\ref{asm:variance_cond}
  is satisfied in LQRs (with linear controllers) with
$\alpha^2 = 3$.  
\end{remark}

Under this assumption, we can get a regret bound with a leading order term dependent on
$J^\star$. The lower order term will depend on a higher
moment version of the information gain, defined as follows:
\begin{eqnarray*}
\gamma_{2,T}(\lambda) & := & 
\max_{\mathcal{A}}  \Exp_\mathcal{A} \left[\left(\log \Big(\textrm{det} \left(\Sigma^{T-1}\right)
/\textrm{det} \left(\Sigma^{0}\right) \Big) \right)^2\right].
\end{eqnarray*} 
Again, for a finite dimensional RKHS with features whose norm bounded is by $B$, then
$\gamma_{2,T}$ will also be $O((d_\phi\log(1+THB^2/\lambda))^2)$ (see
Remark~\ref{remark:info_gain} and Lemma~\ref{lem:finite_info_gain}).

\begin{theorem}[$J^\star$ regret bound] \label{thm:main2}
Suppose that Assumptions~\ref{asm:compute}, \ref{asm:moment_bound},
and \ref{asm:variance_cond} hold and that for all $t$, $J^{\star}(x_0;c^t)\leq J^\star$.
  Again, set $\lambda =\frac{
  \sigma^2}{\|W^\star\|_2^2}$  and define $\widetilde d_T $
as in Theorem~\ref{thm:main1}.
There exist absolute constants $C_1, C_2$ such that if \algname{} (Alg.~\ref{alg:lc3}) is run with
input parameters $C_1$ and $\lambda $, then the following regret bound holds for all $T$,
\begin{align*}
\Exp_{\mathrm{\algname{}}}\left[\textsc{Regret}_T\right]  \leq  
C_2 \  \left(  \alpha J^\star \widetilde d_T\sqrt{ H T}  
+ \alpha H\sqrt{ \Mtwo}\left( \widetilde d_T^{\ 2} + \gamma_{2,T}(\lambda)\right) \right).
\end{align*}
\end{theorem}

See the Discussion (Section~\ref{section:discussion}) for comments
on improving the $ J^\star$ dependence to $\sqrt{ J^\star}$.

\subsection{Proof Techniques}

A key technical, ``self-bounding'' lemma in our proof bounds the difference in cost
under two different models, i.e. $J^\pi(x;c,W^\star) -
J^\pi(x;c,W)$, in terms of the second moment of the
cumulative cost itself, i.e. in terms of $V^\pi(x;c,W^\star)$, where
\[
V^\pi(x;c,W^\star) := \Exp\left[ \left( \sum_{h=0}^{H-1}
c(x_h,u_h) \right)^2 \ \bigg\vert \ x_0=x,\pi,W^\star\right] .
\]

\begin{lemma}[Self-Bounding, Simulation Lemma] 
For any state $x$, non-negative cost function $c$, and model $W$, we have:
\[
J^\pi(x;c,W^\star) - J^\pi(x;c,W)  
\leq  \sqrt{HV^\pi(x;c,W^\star) }\sqrt{ \Exp\left[
%\sum_{h=0}^{H-1}\min\left\{\frac{1}{\sigma^2}\left\|\left(W^\star-W\right)\phi(x_h,u_h) \right\|_2^2,1\right\}
\sum_{h=0}^{H-1}\min\left\{\frac{\left\|\left(W^\star-W\right)\phi(x_h,u_h) \right\|_2^2}{\sigma^2},1\right\}
\right]}
\]
where the expectation is with respect to $\pi$ in $W^\star$ starting
at $x_0=x$.
\end{lemma}

The proof, provided in Appendix~\ref{app:proofs}, involves the construction of a certain stopping time
martingale to handle the unbounded nature of the (realized) cumulative
costs, along with a novel way to handle Gaussian smoothing through
the chi-squared distance function between two distributions. This
lemma helps us in constructing a potential function for the analysis
of the \algname{} algorithm.

\section{Experiments}
\label{sec:sim}
We evaluate \algname{} on three domains: a set of continuous control tasks, a maze environment that requires exploration, and a dexterous manipulation task. 
Throughout these experiments, we use model predictive path integral control (MPPI) \citep{williams2017model} for planning, and
posterior reshaping \citep{chapelle2011empirical} (i.e., scaling of posterior covariance) for Thompson sampling -- we don’t implement \algname{} as analyzed, but rather implement a Thompson sampling variation.
The algorithms are implemented in the Lyceum framework under the Julia programming language \citep{summers2020lyceum,bezanson2017julia}. Comparison algorithms provided by \cite{wang2019exploring,wang2019benchmarking}.
Note that these experiments use reward (negative cost) for evaluations.
Further details of the experiments in this section can be found in Appendix \ref{sec:appsim}.

\subsection{Benchmark Tasks with Random Features}
We use some common benchmark tasks, including MuJoCo \citep{todorov2012mujoco} environments from OpenAI Gym \citep{brockman2016openai}. We use Random Fourier Features (RFF) \citep{rahimi2008random} to represent $\phi$.
Fig. \ref{fig:simpledyn} plots the learning curves against GT-MPPI and the best model-based RL (MBRL) algorithm reported in \cite{wang2019benchmarking}.
It is observed that \algname{} with RFFs quickly increased reward in early stages, indicating low sample complexities empirically.
Table~\ref{tab:finalperf} shows the final performances (at $200$k timesteps) of \algname{} with RFFs for six environments,  
and includes its ranking compared to the benchmarks results from \cite{wang2019benchmarking}.  
We find that \algname{} consistently performs well on simple continuous control tasks, and it works well even without posterior sampling.
However, when the dynamical complexity increases, such as with the contact-rich Hopper model, our method's performance suffers due to non-adaptation of the RFFs. This suggests that more interesting scenarios require different feature representation.

\begin{figure}[t]
	\begin{center}
		\includegraphics[clip,width=\textwidth]{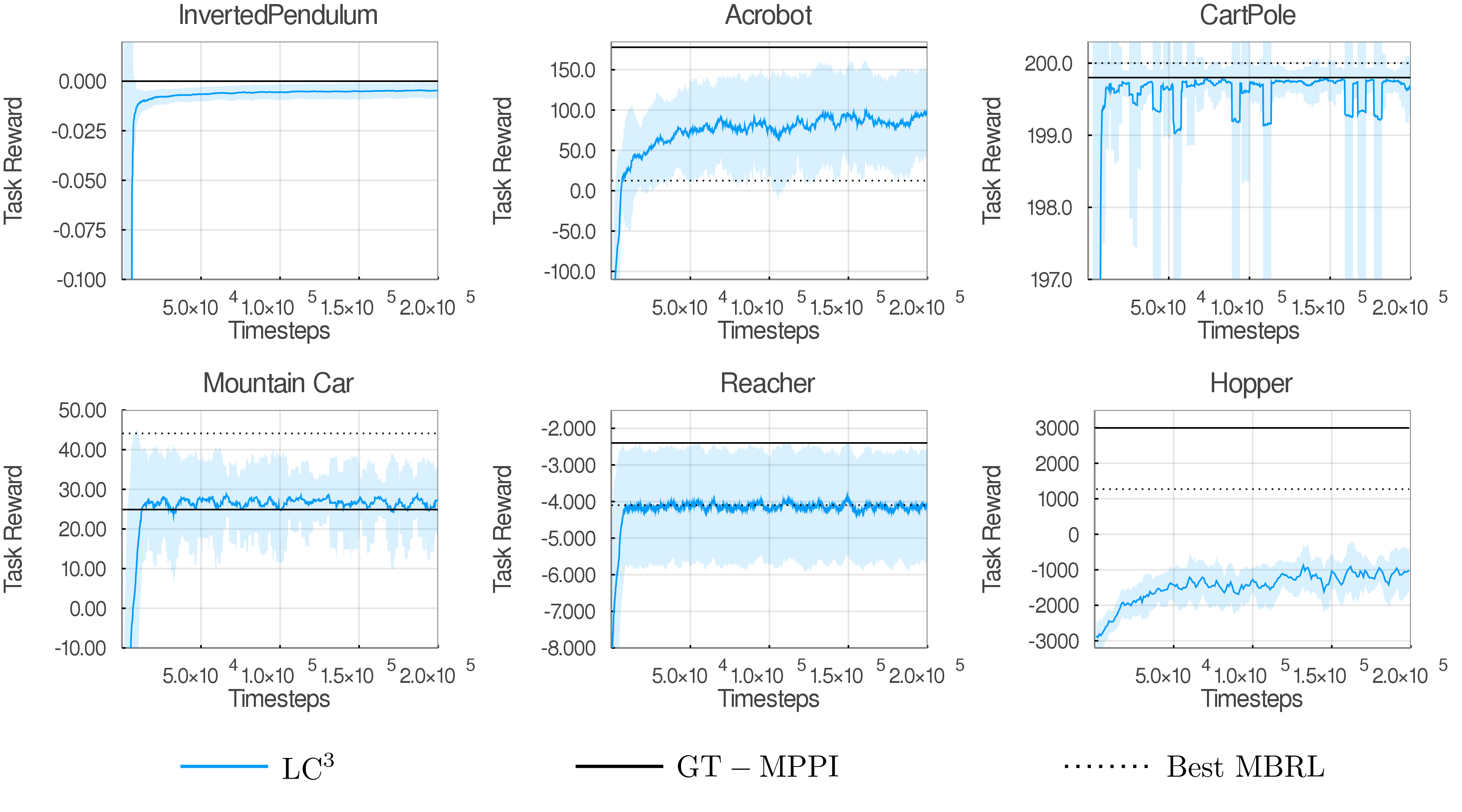}
	\end{center}
	\caption{Performance curves of \algname{} with RFFs for different Gym environments.  Note the reward (negative cost) ranges of those plots are made different.  The final mean performances of GT-MPPI and the best model-based RL (MBRL) algorithm reported in \cite{wang2019benchmarking} are also shown for reference.  The algorithm is run for 200,000 timesteps and with four random seeds.  The curves are averaged over the four random seeds and a window size of 5,000 timesteps.}
	\label{fig:simpledyn}
\end{figure}
%\begin{figure}[t]
%	\begin{minipage}{0.33\hsize}
%		\begin{center}
%			\includegraphics[clip,width=\textwidth]{plot_ip.png}
%			\center{(a) InvertedPendulum}
%		\end{center}	
%	\end{minipage}
%	\begin{minipage}{0.33\hsize}
%		\begin{center}
%			\includegraphics[clip,width=\textwidth]{plot_ab.png}
%			\center{(b) Acrobot}
%		\end{center}
%	\end{minipage}
%	\begin{minipage}{0.33\hsize}
%		\begin{center}
%			\includegraphics[clip,width=\textwidth]{plot_cpg.png}
%			\center{(c) CartPole}
%		\end{center}
%	\end{minipage}
%	\begin{minipage}{0.33\hsize}
%		\begin{center}
%			\includegraphics[clip,width=\textwidth]{plot_mc.png}
%			\center{(d) Mountain Car}
%		\end{center}	
%	\end{minipage}
%	\begin{minipage}{0.33\hsize}
%		\begin{center}
%			\includegraphics[clip,width=\textwidth]{plot_reacher.png}
%			\center{(e) Reacher}
%		\end{center}
%	\end{minipage}
%	\begin{minipage}{0.33\hsize}
%		\begin{center}
%			\includegraphics[clip,width=\textwidth]{plot_hopper.png}
%			\center{(f) Hopper}
%		\end{center}
%	\end{minipage}
%	\caption{Performance curves of \algname{} with RFFs for different Gym environments.  Note the reward (negative cost) ranges of those plots are made different.  The final mean performances of GT-MPPI and the best model-based RL (MBRL) algorithm reported in \cite{wang2019benchmarking} are also shown for reference.  The algorithm is run for 200,000 timesteps and with four random seeds.  The curves are averaged over the four random seeds and a window size of 5,000 timesteps.}
%	\label{fig:simpledyn}
%\end{figure}
\begin{table}[t]
	\centering
	\begin{scriptsize}
		\begin{tabular}{lcccccc}
			\toprule
			& InvertedPendulum & Acrobot & CartPole & Mountain Car  & Reacher & Hopper\\
			\midrule
			\algname{} &$-0.0\pm0.0$ & $95.4\pm52.5$ & $199.7\pm0.4$ & $27.3\pm8.1$ & $-4.1\pm1.6$ & $-1016.5\pm607.4$\\
			%(Without TS) &$-0.005\pm0.004$ & $73.2\pm68.0$ & $199.7\pm0.6$&$26.3\pm4.3$ & $-4.1\pm1.6$  & $-1193.6\pm540.8$\\
			(Ranking) &$1/11$&$1/11$& $2/11$  & $2/11$ & $1/11$  & $7/11$ \\
			\midrule
			GT-MPPI &$-0.0\pm0.0$ & $177.8\pm25.0$ & $199.8\pm0.1$ & $24.9\pm2.9$ & $-2.4\pm0.1$ & $2995.7\pm215.3$\\ 
			PETS-CEM& $-20.5\pm28.9$& $12.5\pm29.0$ & $199.5\pm3.0$ & $-57.9\pm3.6$ & $-12.3\pm5.2$ & $1125.0\pm679.6$ \\
			PILCO& $-194.5\pm0.8$& $-394.4\pm1.4$ & $-1.9\pm155.9$ & $-59.0\pm4.6$ & $-13.2\pm5.9$ & $-1729.9\pm1611.1$ \\
			\bottomrule
		\end{tabular}
	\end{scriptsize}
	\caption{Final performances for six Gym environments. Algorithm are run under the same conditions of \cite{wang2019benchmarking}.  The performances of PETS-CEM and PILCO are copied for reference, and the performance of ground-truth MPPI (GT-MPPI) that has access to true model are also shown.  The results are averaged over four random seeds and a window size of 5,000 timesteps.}
	\label{tab:finalperf}
\end{table}

\subsection{Exploring the Maze}
We construct a maze environment to study the exploration capability of \algname{} (see Fig.~\ref{fig:maze} (a)). 
State and control take values in $[-1,1]^2\subset\R^2$ and in $[-1,1]\subset\R$, respectively.
The task is to bring an agent to the goal state being guided by  
the negative cost (reward) $- c(x_h,u_h) = 8 - \|x_{h}-[1,1]^{\top}\|_2^2$. 
We use a one-hot vector of states and actions as features.
%We use a one-hot vector of states and actions (e.g., $\phi(x,u)=[1,0,\ldots,0]^{\top}$ if $x\leq-0.75$ and $u\leq-0.5$) as features.

We compare the performances, over $50$ episodes with task trajectory length $30$, of \algname{} (with different scale parameters for posterior reshaping) to random walk and PETS-CEM.
Fig.~\ref{fig:maze} (b) plots the means and standard deviations, across four random seeds, of the number of state-action pairs visited over episodes.
%Across four random seeds, the success rate (i.e., the number of runs in which the agent reaches the goal within $50$ episodes per the total number of runs) of the best setting of \algname{} was $1.0$, while those for random walk and PETS-CEM were $0.0$.  For the best setting of \algname{}, the average number of episodes required for the first success was $25.0$.
We observe that \algname{}'s strategic exploration better modeled the setting for higher rate of success.  

%To scale up the problem to more complicated robotics applications, we need well-structured representations as discussed next.
%Refer to Appendix~\ref{sec:appsim} for the details of this environment and parameter setup.
\begin{figure}[t]
	\begin{center}
		\includegraphics[clip,width=\textwidth]{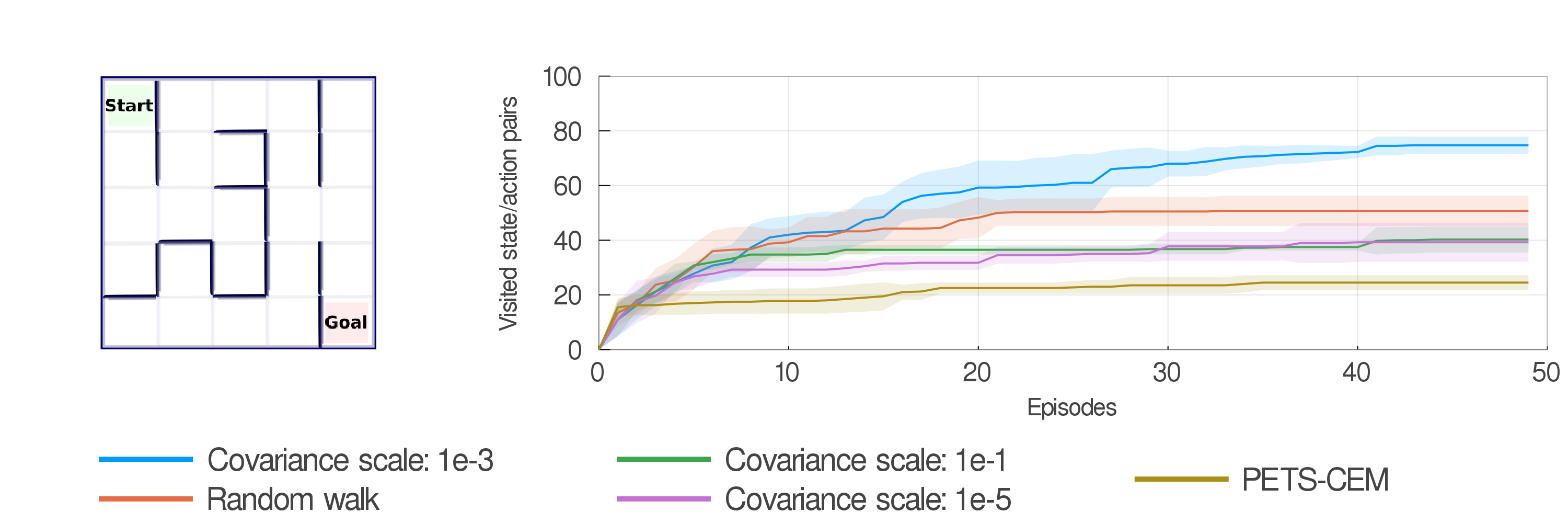}
	\end{center}
	\caption{Left: An illustration of the maze environment.  Start and End states are $[-1, -1]^{\top}$ and $[1, 1]^{\top}$, respectively.  Dark lines are ``walls''.  Right: The means and standard deviations, across four random seeds, of the number of state-action pairs already explored over episodes.  Covariance scale is the posterior reshaping constant of Thompson sampling.  Random walk takes actions uniformly sampled within $[-1,1]$.  PETS-CEM is a representative model-based RL which uses uncertainty of dynamics but without exploration.  The agent always reaches the goal within $50$ episodes under the best setting of \algname{} and the average number of episodes required for the first success is $25.0$, while random walk and PETS-CEM never bring the agent to the goal within $50$ episodes.} 
	\label{fig:maze}
\end{figure}

\subsection{Practical Application}
As we might consider learning model dynamics for the real world in applications such as robotics, we need sufficiently complex features -- without resorting to large scale data collection for feature learning.  One solution to this problem is creating an ensemble of parametric models, such as found in \cite{Tobin17, mordatch2015ensemble}.  We take the perspective that most model parameters of a robotic system will be known, such as kinematic lengths, actuator specifications, and inertial configurations.  Since we would like robots to operate in the wild, some dynamical properties may be unknown: in this case, the manipulated object's dynamical properties.  Said another way, the robot knows about itself, but only a little about the object.

In this experiment, we demonstrate our model learning algorithm on a robotics inspired, dexterous manipulation task.
An arm and hand system (see Fig.~\ref{fig:cooldemo}) must pick up a spherical object with unknown dynamical properties, and hold it at a target position. The entire system has 33 degrees of freedom, and an optimal trajectory would involve numerous discontinuous contacts; the system dynamics are not well captured by random features and such features are not easily learned. We instead use the predictive output of an ensemble of six MuJoCo models as our features $\phi$, each with randomized parameters for the object. Using a single model from the ensemble with the planner is unable to perform the task.

Fig.~\ref{fig:cooldemo} plots the learning curves of \algname{} with different features.  We observe that, within 10 attempts at the task, \algname{} with ensemble features is successful, while the same method with RFF features makes no progress.  Additionally, we use \algname{} with the top layers of a neural network -- trained on a data set of $30$ optimized trajectories with the correct model -- as our features.  It also makes little progress.

We clarify the setting in which this approach may be relevant as follows. Complex dynamics, such as that in the real world, are difficult to represent with function approximation like neural networks or random features. %This problem can be broken into two parts: learning the structure of the features, and how to combine the features into a model. 
Rather than collect inordinate amounts of data to mimic the combination of features and model, we instead use structured representations of the real world, such as dynamics simulators, to produce features, and use the method in this work to learn the model.  Since dynamics simulators represent the current distillation of physics into a computational form and accurate measurement of engineered systems is paramount for the modern world, this instantiation of this method is reasonable.

\begin{figure}[t]
% 	\begin{minipage}{0.2\hsize}
% 		\begin{center}
% 			\includegraphics[width=\textwidth]{figures/armhand.png}
% 			\center{(a)}
% 		\end{center}	
% 	\end{minipage}
% 	\begin{minipage}{0.8\hsize}
% 		\begin{center}
% 			\includegraphics[width=\textwidth]{figures/plot_armhand_wide.png}
% 			\center{(b)}
% 		\end{center}
% 	\end{minipage}
	\includegraphics[width=\textwidth]{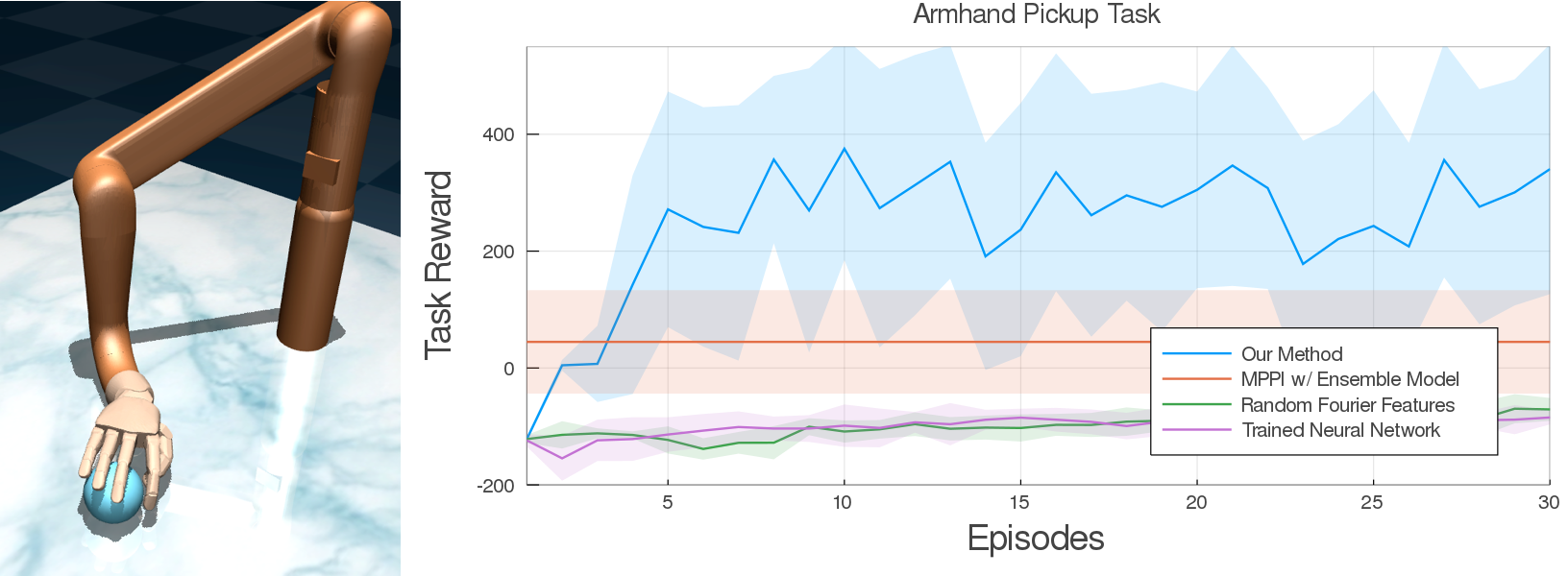}
	\caption{Left: An illustration of Armhand environment.  Right: Performance curves averaged across 12 random seeds.  For reference, we include the average reward of MPPI using a random model from the ensemble. Its score represents the system moving the hand to the object, but unable to grasp and lift it: exactly what we would expect for randomized object dynamics parameters.}
	\label{fig:cooldemo}
\end{figure}

\section{Discussion}\label{section:discussion}
%This work provides $O(\sqrt{T})$ regret bounds for the online
%nonlinear control problem, where we utilize a number of analysis
%concepts from reinforcement learning and machine learning for
%continuous problems. There are a number of important future directions.
%
%\paragraph{Lower bounds:} Sharp lower bounds would be important to
%develop for this very natural model. We conjecture with stronger
%assumptions on higher order moments that an optimal
%$O(\sqrt{H^2T})$ regret is achievable.
%
%\paragraph{The single trajectory case:} 
%It would be interesting to use these techniques to develop
%regret bounds for the single trajectory case, under 
%stronger stability and mixing assumptions of the process
%(see~\cite{cohen2019learning} for the LQR case).
%
%
%\paragraph{Feature learning:} As of now, we have assumed the RKHS is
%known. A practically relevant direction would be to learn a good
%feature space.

This work provides $O(\sqrt{T})$ regret bounds for the online
nonlinear control problem, where we utilize a number of analysis
concepts from reinforcement learning and machine learning for
continuous problems.
Though our work focuses on
the theoretical foundations of learning in nonlinear control, we believe our work
has broader impact in the following aspects.

Our work helps to further connect two communities: Reinforcement
Learning Theory and Control Theory. Existing models considered in RL
literature that have provable guarantees rarely are applicable to
continuous control problems while only few existing control theory
results focus on the (non-asymptotic) sample complexity aspect of
controlling unknown dynamical systems. Our work demonstrates that
a natural nonlinear control model, the KNR model, is learnable from a
learning theoretical perspective. 

\iffalse
While it seems that two communities
are more separated than would be ideal, we believe our work paves a
new way for further communication between two communities.
\fi

From a practical application perspective, the sample efficiency of our
algorithm enables control in complex dynamical settings without
onerous large scale data collection, hence demonstrates potentials for
model learning and control in real world applications such as
dexterous manipulation, medical robotics, human robot interaction, and
self-driving cars where complicated nonlinear dynamics are involved
and data is often extremely expensive to collect.  

\iffalse
Also, this line of work may be helpful to provide new means for handling model uncertainty in
nonlinear planning,  which may be relevant in the broader context of
safety and reliability. A clear caveat here is understanding the role
of model misspecification.
\fi

Lastly, there are a number of important extensions and future directions.

\paragraph{Lower bounds:} Sharp lower bounds would be important to
develop for this very natural model. As discussed in
Remark~\ref{remark:dimension}, our results are already minimax optimal for
some parameter scalings.

\paragraph{Improved upper bounds \& $J^\star$ vs $\sqrt{J^\star}$
  dependencies:} We conjecture with stronger assumptions on higher
order moments that an optimal $O(\sqrt{H^2T})$ regret is achievable.
It is also plausible that with further higher moment assumptions then,
for the stabilizing case, the dependence on $J^\star$ can be improved
to $\sqrt{J^\star}$.  Here, our conjecture is that one would, instead,
need to make a boundedness assumption on the ``index
of dispersion,'' i.e., that the ratio of the variance to the mean is bounded; we
currently assume the ratio of the standard deviation to the mean is
bounded.

\paragraph{The single trajectory case:}
It would be interesting to use these techniques to develop
regret bounds for the single trajectory case, under
stronger stability and mixing assumptions of the process
(see~\cite{cohen2019learning} for the LQR case).

\paragraph{Feature learning:} As of now, we have assumed the RKHS is
known. A practically relevant direction would be to learn a good
feature space.

\section*{Acknowledgments}
The authors wish to thank Horia Mania for graciously sharing his
thoughts in this line of work.
Motoya Ohnishi thanks Aravind Rajeswaran, Vikash Kumar, and Ben Evans
at Movement Control Laboratory for valuable discussions on model-based
RL.  Also, he thanks Colin Summers for instructions on Lyceum.
Kendall Lowrey and Motoya Ohnishi thank Emanuel Todorov for valuable
discussions and Roboti LLC for computational supports.
Sham Kakade acknowledges funding from the Washington Research
Foundation for Innovation in Data-intensive Discovery, the ONR award
N00014-18-1-2247, NSF Award CCF-1703574, and the NSF Award
CCF-1637360.  Motoya Ohnishi was supported in part by Wissner-Slivka Endowed Fellowship and Funai Overseas Scholarship.

\bibliography{ref}
\bibliographystyle{plainnat}
\clearpage

\appendix

\section{Additional Notation}

A controller is a mapping $\pi: \mathcal{X}\times \{0,\ldots H-1\} \rightarrow
\mathcal{U}$.   Given a instantaneous cost function $c:\mathcal{X} \times
\mathcal{U}\rightarrow \R$, we define the cost (or the ``cost-to-go'' ) of a policy as:
\[
J^\pi(x;c,W) =  \Exp \left[ \sum_{h=0}^{H-1} c(x_h, u_h)  \Big|
  \pi, x_0=x, W
\right] 
\label{eq: cost}
\]
where the expectation is under trajectories sampled under $\pi$
starting from $x_0$ in model parameterized by $W$.
The ``cost-to-go'' at state $x$ at time $h \in \{0,\ldots H-1\} $ is denoted by:
\[
J^\pi_h(x;c,W)  = \Exp \left[ \sum_{\ell=h}^{H-1} c(x_\ell, u_\ell)  \Big|
  \pi, x_h=x
\right]  .
\]

When clear from context, we let the episode $t$ index the
policy, e.g. we write $J^{t}(x;c)$ to refer to $J^{\pi^t}(x,c)$.
Subscripts refer to
the timestep within an episode and superscripts index the episode
itself, i.e. $\phi^t_h$ will refer to the random vector which is the
observed features during timestep $h$ within episode $t$.  We let
$\mathcal{H}_t$ denote the history up to the beginning of episode $t$.

Also, $\|x\|_M^2 := x^\top M x$ for  a vector $x$ and a matrix $M$.

%For a matrix $M$, we use $\|M\|_2$ and $\|M\|_F$ to denote the
%spectral and Frobenius norms.
 
\section{Lower Confidence Bound based Analysis} 
\label{app:proofs}

In this section, we provide proofs for the two main theorems: Theorem~\ref{thm:main1} and Theorem~\ref{thm:main2}.

\subsection{Simulation Analysis}\label{section:simulation}

We derive a novel self-bounding simulation lemma
(Lemma~\ref{lem:self_bound}) in this section, using the Optional
Stopping Theorem.

\begin{lemma}[Difference Lemma]\label{lem:traj_wise_diff} Fix a policy
  $\pi$, cost function $c$, and model $W$.
Consider any trajectory $\{x_h, u_h\}_{h=0}^{H-1}$ where $u_h =
\pi(x_h)$ for all $h \in \{0, \ldots H-1\}$. For $h \in \{0, \ldots
H-1\}$, let $\widehat{J}_h$ refer to the realized cost-to-go on this
trajectory, i.e.
\[
\widehat{J}_h = \sum_{\tau = h}^{H-1} c(x_\tau, u_\tau) .
\]
For all $\tau\in \{1, \ldots
  H-1\}$, we have that:
\begin{align*}
&\widehat{J}_0 - J_0^\pi(x_0; c, W) = \widehat{J}_\tau -
  \Exp_{ x^\prime_{\tau}\sim P(\cdot| W, x_{\tau-1},u_{\tau-1})}
  J^{\pi}_\tau(x^\prime_\tau;c, W) \\
& \qquad \qquad + \sum_{h=1}^{\tau - 1} J^{\pi}_{h}(x_h; c, W) -
  \Exp_{x^\prime_{h}\sim P(\cdot|W, x_{h-1},u_{h-1})}
  J^{\pi}_{h}(x^\prime_h; c,W)
%\widehat{J}_{\tau} -  J^{\pi}_\tau(x_\tau; c, W) \\
%& \qquad +  \sum_{i = \ell}^{\tau-1} J^{\pi}_{i+1}(x_{i}; c, W) - \Exp_{x_{i+1}\sim P(\cdot| W, x_{i},u_{i})} J^{\pi}_{i}(x_{i}; c, W)
\end{align*}
\end{lemma}

\begin{proof}
Starting from $h = 0$, using $u_0 = \pi(x_0)$, we have:
\begin{align*}
&\widehat{J}_0 - J_0^\pi(x_0; c, W) =  \widehat{J}_{1} -
  \Exp_{x^\prime_{1} \sim P(\cdot | W, x_0, u_0)}
  J_{1}^\pi(x^\prime_{1}; c, W) \\
& = \widehat{J}_{1} - J^\pi_{1}(x_{1}; c, {W}) + J^{\pi}_{1}(x_{1}; c,
  W) - \Exp_{x^\prime_{1} \sim P(\cdot | W, x_0, u_0)}
  J_\ell^\pi(x^\prime_{1}; c, W)\\
& = \widehat{J}_{2} - \Exp_{x^\prime_{2}\sim
  P(\cdot|x_{1},u_{1},W)}J^{\pi}_{2}(x^\prime_{2}; c, W) \\
& \qquad + J^{\pi}_{1}(x_{1}; c, W)  - \Exp_{x^\prime_{1} \sim
  P(\cdot | W, x_0, u_0)} J_1^\pi(x^\prime_{1}; c, W).
\end{align*} 
Recursion completes the proof, 
where, at each step of the recursion,  we add and subtract $J_t^{\pi}(x_t;c, W)$ and apply the same
operation on the term $\widehat{J}_t - J^{\pi}_t(x_t;c, W)$; 
%where we can apply the same operation on the term $\widehat{J}_{2} - J^{\pi}_\ell(x_{2}; c, W)$ recursively. This concludes the proof. 
\end{proof}

\begin{lemma}[``Optional Stopping'' Simulation Lemma] \label{lemma:stop}
Fix a policy
  $\pi$, cost function $c$, and model $W$.
Consider the stochastic process over trajectories, where $\{
x_h,u_h\}_{h=0}^{H} \sim \pi$ is sampled with respect to the 
model $W^\star$. With respect to this stochastic process, define a stopping time $\tau$ as:
\begin{align*}
\tau = \min\left\{ h\geq 0: J_h^{\pi}(x_h; c, W) \geq J_h^{\pi}(x_h; c, W^\star) \right\}.
\end{align*} 
Define the random variable $\widetilde{J}^{\pi}_h(x_h)$ as:
\begin{align*}
\widetilde{J}^{\pi}_h(x_h) = \min\left\{  J^{\pi}_h(x_h; c, W), J^{\pi}_h(x_h; c, W^\star)  \right\}.
\end{align*} 
%for any $h$ and $x_h$.
We have that:
\begin{align*}
&J^{\pi}_0(x_0; c, W^\star) - J^{\pi}_{0}(x_0; c, W) \\
&\leq \Exp\left[  \sum_{h=0}^{H-1} \mathrm{1}\{  h < \tau  \} \left( \Exp_{x^\prime_{h+1}\sim P(\cdot| W^\star, x_h,u_h)} \widetilde{J}^{\pi}_{h+1}(x^\prime_{h+1}) 
- \Exp_{x^\prime_{h+1}\sim P(\cdot| W, x_h,u_h)} \widetilde{J}^{\pi}_{h+1}(x^\prime_{h+1})   \right)   \right]
\end{align*}
where the expectation is with respect to $\{
x_h,u_h\}_{h=0}^{H} \sim \pi$ sampled with respect to the model $W^\star$.
\end{lemma}

\begin{proof}
Our filtration, $\mathcal{F}_h$,  at time $h$ will be the previous
noise variables, i.e.
\[
\mathcal{F}_{h}:=\{\eps_0, \eps_1 \ldots \eps_{h-1} \},
\]
and note that $\{ x_1, u_1, c(x_1, u_1), \ldots x_h, u_h, c(x_h, u_h) \}$ is fully determined by $\mathcal{F}_{h}$.
Also, observe that $\tau$ is a valid stopping time with respect to the
filtration $\mathcal{F}_h$.

Define:
\[
M_h = \Exp\left[
\widehat J_0 - J^{\star}(x_0; c, W ) \mid
\mathcal{F}_{h} \right] 
\]
which is a Doob martingale (with respect to our filtration), and so 
$\Exp[M_{h+1}|\mathcal{F}_{h}]= M_h$.
%$\Exp[M_h]= \Exp\left[ \widehat J_0 - J^{\star}(x_0; c, W )  \right]$.
By Doob's optional stopping theorem,
\begin{equation}\label{eq:Doob}
\Exp\left[
\widehat J_0 - J^{\star}(x_0; c, W ) 
\right] 
= \Exp[M_\tau] =
\Exp\left[  \Exp\left[
\widehat J_0 - J^{\star}(x_0; c, W ) \mid
\mathcal{F}_{\tau} \right]  \right].
\end{equation}
The proof consists in bounding $M_\tau$.

Consider an $\mathcal{F}_{\tau}$, which is stopped at the random
time $\tau$. By Lemma~\ref{lem:traj_wise_diff}, 
\begin{eqnarray*}
M_\tau & =&\Exp\left[
\widehat{J}_0 - J^{\star}(x_0;c, W ) \mid
\mathcal{F}_{\tau} \right]\\
&= & J_\tau\left( x_\tau;c, W^\star \right) -
\Exp_{x^\prime_\tau \sim P(\cdot|W,x _{\tau-1},u _{\tau-1})}J_h(x^\prime_\tau; c, W ) \\
&&+ \sum_{h=1}^{\tau-1}\Big(J_h(x_h; c, W) 
-\Exp_{x^\prime_h \sim P(\cdot|W,x _{h-1},u _{h-1})}J_h(x^\prime_h; c, W ) \Big)\\
&= & 
\sum_{h=1}^{\tau} \Big(\widetilde J_h\left( x_h\right) - 
\Exp_{x^\prime_h \sim P(\cdot|W,x_{h-1},u_{h-1})} 
J_h(x^\prime_h; c,W ) \Big)\\
&\leq & 
\sum_{h=1}^{\tau} \Big(\widetilde J_h\left( x_h\right) - 
\Exp_{x^\prime_h \sim P(\cdot|W,x_{h-1},u_{h-1})} 
\widetilde J_h\left( x^\prime_h\right) \Big)\\
&= & 
\sum_{h=1}^{H} \mathrm{1}(h\leq \tau)\Big(\widetilde J_h\left( x_h\right) - 
\Exp_{x^\prime_h \sim P(\cdot|W,x_{h-1},u_{h-1})} 
\widetilde J_h\left(x^\prime_h\right) \Big).
\end{eqnarray*}
where the third equality follows using the definition of $\tau$ which
implies that  $J_\tau\left( x_\tau; c, W^\star
\right) =\widetilde J_\tau\left( x_\tau\right)$ 
and that $J_h(x_h; c, W ) =\widetilde J_h\left( x_h\right)$ for $h< \tau$;  and the inequality is due to the definition
of $\widetilde J$. 

Using this bound on $M_\tau$ and Equation~\ref{eq:Doob}, we have:
\begin{eqnarray*}
\Exp\left[
\widehat J_0 - J^{\star}(x_0; c, W ) 
\right] 
\leq
\sum_{h=1}^{H}\Exp\left[ \mathrm{1}(h\leq \tau)\Big(\widetilde J_h\left( x_h\right) - 
\Exp_{x^\prime_h \sim P(\cdot|W,x_{h-1},u_{h-1})} 
\widetilde J_h\left( x^\prime_h\right) \Big)  \right].
\end{eqnarray*}
\iffalse
\begin{eqnarray*}
&&\Exp\left[
\widehat J_0 - J^{\star}(x_0; c, W ) 
\right] \\
&=&
\Exp\left[  \Exp\left[
\widehat J_0 - J^{\star}(x_0; c, W ) \mid
\mathcal{F}_{\tau} \right]  \right]\\
&\leq&
\sum_{h=1}^{H}\Exp\left[ \mathrm{1}(h\leq \tau)\Big(\widetilde J_h\left( x_h\right) - 
\Exp_{x^\prime_h \sim P(\cdot|W,x_{h-1},u_{h-1})} 
\widetilde J_h\left( x^\prime_h\right) \Big)  \right].
\end{eqnarray*}
\fi
For the $h$-th term, observe:
\begin{eqnarray*}
&&
\Exp\Big[ \mathrm{1}(h\leq \tau)\Big(\widetilde J_h\left( x_h\right) - 
\Exp_{x^\prime_h \sim P(\cdot|W,x_{h-1},u_{h-1})} 
\widetilde J_h\left( x^\prime_h\right) \Big) \ \Big]\\
&=&\Exp\Big[ \Exp\left[
\mathrm{1}(h\leq \tau)\Big(\widetilde J_h\left( x_h\right) - 
\Exp_{x^\prime_h \sim P(\cdot|W,x_{h-1},u_{h-1})} 
\widetilde J_h\left( x^\prime_h\right) \Big) \mid \mathcal{F}_{h-1} \right] \Big]\\
&=&\Exp\Big[ \Exp\left[
\mathrm{1}(h-1< \tau)\Big(\widetilde J_h\left( x_h\right) - 
\Exp_{x^\prime_h \sim P(\cdot|W,x_{h-1},u_{h-1})} 
\widetilde J_h\left( x^\prime_h\right) \Big) \mid \mathcal{F}_{h-1} \right] \Big]\\
&=&\Exp\Big[  \mathrm{1}(h-1< \tau)\Exp\left[
\widetilde J_h\left( x_h\right) - 
\Exp_{x^\prime_h \sim P(\cdot|W,x_{h-1},u_{h-1})} 
\widetilde J_h\left( x^\prime_h\right) \mid \mathcal{F}_{h-1} \right] \Big]\\
&=&\Exp\Big[  \mathrm{1}(h-1< \tau) \left(
\Exp_{x^\prime_{h} \sim P(\cdot|W^\star,x_{h-1},u_{h-1}) } \widetilde J_{h }(x^\prime_{h}) - 
\Exp_{x^\prime_{h} \sim P(\cdot|W,x_{h-1},u_{h-1})} 
\widetilde J_{h} (x^\prime_{h})\right) \Big].
\end{eqnarray*}
where the second equality uses that $\mathrm{1}(h\leq \tau) =\mathrm{1}(h-1<
\tau)$,  and the third equality uses that $\mathrm{1}(h-1< \tau)$ 
%is fully determined by $\mathcal{F}_{h-1} $, i.e. $\mathrm{1}(h-1< \tau)$
is measurable with respect to $\mathcal{F}_{h-1}=\{\eps_0, \ldots \eps_{h-2} \}$.
This completes the proof.
\end{proof}

%\newpage
%\input{app_stop.tex}
%\newpage

The previous lemma allows us to bound the difference in cost
under two different models, i.e. $J^\pi(x;c,W^\star) -
J^\pi(x;c,W)$, in terms of the second moment of the
cumulative cost itself, i.e. in terms of $V^\pi(x;c,W^\star)$, where
\[
V^\pi(x_0;c,W^\star) := \Exp\left[ \left( \sum_{h=0}^{H-1}
c(x_h,u_h) \right)^2 \mid x_0,\pi, W^\star \right] .
\]
%where the expectation is under the $\pi$ in the model $W^\star$.

\begin{lemma}[Self-Bounding, Simulation Lemma] \label{lem:self_bound}
For any policy $\pi$, model parameterization $W$, and non-negative cost
$c$, and for any state $x_0$,  we have:
\begin{align*}
&J^\pi(x_0;c,W^\star) - J^\pi(x_0;c,W)  \\
& \leq  \sqrt{HV^\pi(x_0;c,W^\star) }\sqrt{ \Exp\left[
\sum_{h=0}^{H-1}\min\left\{\frac{1}{\sigma^2}\left\|\left(W^\star-W\right)\phi(x_h,u_h) \right\|_2^2,1\right\}
%\min\left\{\sum_{h=0}^{H-1}\left\|\phi(x_h,u_h) \right\|^2_{\Sigma^{-1}},1\right\}
\right]}.
\end{align*}  
where the expectation is with respect to $\pi$ in $W^\star$ starting
at $x_0$.
\end{lemma}

\begin{proof}
For the proof, it is helpful to define the random variables:
\begin{eqnarray*}
\Delta_h &=& \Exp_{x^\prime_{h+1} \sim P(\cdot|W^\star,x_{h},u_{h}) } \left[\widetilde J_{h+1 }(x^\prime_{h+1}) \right]- 
\Exp_{x^\prime_{h+1} \sim P(\cdot|W,x_h,u_h)} \left[ \widetilde J _{h+1} (x^\prime_{h+1}) \right]\\
A_h &:=& \Exp_{x^\prime_{h+1} \sim P(\cdot|W^\star,x_{h},u_{h}) } \left[\widetilde J_{h+1 }(x^\prime_{h+1})^2 \right]
\end{eqnarray*}

By Lemma~\ref{lem:mean_difference} (which bounds the difference in means under
two Gaussian distributions, using the chi-squared distance function),
we have:
\begin{align*}
&\Delta_h \leq \sqrt{\Exp_{x_{h+1} \sim P(\cdot|W^\star,x_{h},u_{h}) } \left[\widetilde J_{h+1 }(x_{h+1})^2 \right]}
\min\left\{\frac{1}{\sigma}\left\|\left(W^\star-W\right)\phi(x_h,u_h) \right\|_2,1\right\}\\
&= \sqrt{A_h}\min\left\{\frac{1}{\sigma}\left\|\left(W^\star-W\right)\phi(x_h,u_h) \right\|_2,1\right\}.
\end{align*}

From Lemma~\ref{lemma:stop}, we have:
\begin{eqnarray*}
%&&J_0(x_0) - \widetilde J_{0 }(x_0) \leq \sum_{h=0}^{H-1} \Exp\left[\mathrm{1}(h< \tau) \Delta_h\right] \nonumber\\
&&J^{\pi}_0(x_0; c, W^\star) -  J^{\pi}_{0 }(x_0; c, W) \leq \sum_{h=0}^{H-1} \Exp\left[\mathrm{1}(h< \tau) \Delta_h\right] \nonumber\\
&\leq& \sum_{h=0}^{H-1} \Exp\left[\sqrt{A_h}\min\left\{\frac{1}{\sigma}\left\|\left(W^\star-W\right)\phi(x_h,u_h) \right\|_2,1\right\}\right]\\
&\leq &
\sum_{h=0}^{H-1} \sqrt{\Exp\left[A_h\right]}
\sqrt{\Exp\left[
\min\left\{\frac{1}{\sigma^2}\left\|\left(W^\star-W\right)\phi(x_h,u_h) \right\|_2^2,1\right\}
\right]} \\
&\leq& \sqrt{\Exp\left[\sum_{h=0}^{H-1} A_h\right]} \sqrt{\Exp\left[\sum_{h=0}^{H-1}
\min\left\{\frac{1}{\sigma^2}\left\|\left(W^\star-W\right)\phi(x_h,u_h) \right\|_2^2,1\right\}\right]},
\end{eqnarray*} where in the second inequality we use $\Exp[ab] \leq\sqrt{ \Exp[a^2]  \Exp[b^2]}$ and the Cauchy-Schwartz inequality in the last inequality. 
For the first term, observe that:
\begin{align*}
\Exp\left[ A_h\right]
&=\Exp\left[ \Exp_{x^\prime_{h+1} \sim P(\cdot|W^\star,x_{h},u_{h}) } \left[\widetilde J_{h+1 }(x^\prime_{h+1})^2 \right]\right]
= \Exp\left[\widetilde J_{h+1 }(x_{h+1})^2 \right]
\leq \Exp\left[ J_{h+1 }(x_{h+1})^2 \right] \\
&= \Exp\left[ \left( \Exp\left[ 
\sum_{\ell=h+1}^{H-1}c(x_\ell,u_\ell) \mid x_{h+1}\right] \right)^2 \right] 
\leq \Exp\left[ 
 \left(\sum_{\ell=h+1}^{H-1}c(x_\ell,u_\ell) \right)^2 \right]\\
&\leq \Exp\left[ \left( \sum_{\ell=0}^{H-1} c(x_\ell,u_\ell)\right)^2 \right] = V^\pi
\end{align*}
where the first inequality uses the definition of $\widetilde J $; the
second inequality follows from Jensen's inequality; and the last
inequality follows from our assumption that the instantaneous costs
are non-negative.  The proof is completed by substitution.
\end{proof}

\subsection{Regret Analysis (and proofs of Theorem~\ref{thm:main1} and Theorem~\ref{thm:main2})}

Throughout, let $\mathcal{E}_{t,cb}$ be the event that  $W^\star \in
\textsc{Ball}^{t}$ holds at episode $t$.

\begin{lemma}[Per-episode Regret Lemma]
\label{lem:regret} Suppose Assumptions~\ref{asm:compute} and~\ref{asm:moment_bound} hold.
Let $\mathcal{H}_{<t}$ be the history of events
before episode $t$. For the \algname{}, we have:
\begin{align*}
&\mathbf{1}({\mathcal{E}}_{t,cb}) \Big( J^t(x_0;c^t) - J^\star(x_0;c^t)\Big)\\
& \leq  \sqrt{HV^{^t}(x_0;c,W^\star) \left(\frac{4\beta^t}{\sigma^2}+ H  \right)}\sqrt{ \Exp\left[
\min\left\{ \sum_{h=0}^{H-1} \|\phi_h^t\|^2_{(\Sigma^t)^{-1}} ,1\right\}
\ \Big\vert \ \mathcal{H}_{<t} \right]}.
\end{align*}
Note that the expectation is with respect to the trajectory of \algname{},
i.e. it is under $\pi^t$ in $W^\star$.
\end{lemma}

\begin{proof}
Suppose $\mathcal{E}_{cb}^t$ holds, else the lemma is
immediate. By construction of the \algname{} algorithm (the optimistic
property) and by the self-bounding, simulation lemma
(Lemma~\ref{lem:self_bound}), we have:
\begin{align*}
&J^t(x_0;c^t, W^\star) - J^\star(x_0; c^t, W^\star) \leq J^t(x_0;c^t, W^\star) - J^t(x_0; c^t, \widehat{W}^t) \\
& \leq  \sqrt{HV^{^t}(x_0;c,W^\star) }\sqrt{ \Exp\left[
\sum_{h=0}^{H-1}\min\left\{\frac{1}{\sigma^2}\left\|\left(W^\star-\widehat{W}^t\right)\phi_h^t \right\|_2^2,1\right\}
\ \Big\vert \ \mathcal{H}_{<t} \right]}.
\end{align*} 
where the expectation is with respect to the trajectory of \algname{},
i.e. of $\pi^t$ in $W^\star$.

For  $W^\star \in \textsc{Ball}^{t}$, we have
\begin{align*}
&\left\| \left(\widehat{W}^t - W^\star\right)\phi_h^t \right\|_2 
\leq  
\left\|\left(\widehat{W}^t - W^\star\right) (\Sigma^t)^{1/2}\right\|_2 \left\| (\Sigma^t)^{-1/2}\phi_h^t \right\|_2\\
 & \leq  
 \Big(\left \|\left(\widehat{W}^t - \overline{W}^t\right)(\Sigma^t)^{1/2}\right\|_{2} 
+\left \|\left(\overline{W}^t - W^\star\right) (\Sigma^t)^{1/2}\right\|_{2} \Big)
\left\| \phi_h^t \right\|_{(\Sigma^t)^{-1}}
 \leq 2\sqrt{{\beta^t}}  \|\phi_h^t\|_{(\Sigma^t)^{-1}}.
\end{align*} 
where we have also used that  $\widehat{W}^t, \overline{W}^t \in
\textsc{Ball}^{t}$, by construction.

This implies that:
\begin{align*}
&\sum_{h=0}^{H-1}\min\left\{\frac{1}{\sigma^2} \|(W^\star - \widehat{W}^t)\phi_h^t\|_2^2  , 1   \right\} \leq \sum_{h=0}^{H-1} \min\left\{ \frac{4\beta^t}{\sigma^2}\|\phi_h^t\|^2_{(\Sigma^t)^{-1}}  , 1   \right\} \\
&\leq \min\left\{ \frac{4\beta^t}{\sigma^2} \sum_{h=0}^{H-1} \|\phi_h^t\|^2_{(\Sigma^t)^{-1}}  ,H\right\}  
\leq \max\left\{ \frac{4\beta^t}{\sigma^2}, H  \right\} \min\left\{ \sum_{h=0}^{H-1} \|\phi_h^t\|^2_{(\Sigma^t)^{-1}} ,1\right\}.
\end{align*} 
The proof is completed by substitution.
\end{proof}

Before we complete the proofs, the following two lemmas  are
helpful. Their proofs are provided in
Appendix~\ref{section:confidence}. The first lemma 
bounds the sum failure probability of $W^\star$ not being in all the confidence
balls (over all the episodes);  the lemma  generalizes the
argument from \citep{abbasi2011improved,dani2008stochastic} to matrix regression.

\begin{lemma}[Confidence Ball] Let
\[
\beta^t =  2{\lambda}\|W^\star\|_2^2+ 
8\sigma^2\left( d_{\mathcal{X}} \log(5) +  2\log(t) + \log(4) + \log\left( \det(\Sigma^t)/\det(\Sigma^0) \right) \right).
%\beta^t =  2{\lambda}\|W^\star\|_2^2+ 16\sigma^2\left( d_{\mathcal{X}} \log(5) +  2\log(t) + \log(2) + \log\left( \det(\Sigma^t)/\det(\Sigma^0) \right) \right).
\] 
We have:
\begin{align*}
\sum_{t=0}^{\infty} \mathrm{Pr}\left( \overline{\mathcal{E}}_{t,cb} \right) =
\sum_{t=0}^{\infty} \mathrm{Pr}\left( 
\left\| \left(\overline{W}^t - W^\star\right) \left(\Sigma^t\right)^{1/2} \right\|^2_{2} >  \beta^t
  \right) 
\leq \frac{1}{2}.
\end{align*}
\label{lem:confidence_ball}
\end{lemma}

The next lemma provides a bound on the potential function used in our
analysis. It is based on the elliptical potential function
argument from \citep{dani2008stochastic,srinivas2009gaussian}.

\begin{lemma}[Sum of Potential Functions] For any sequence of
  $\phi^t_h $, we have:
\begin{align*}
\sum_{t=0}^{T-1} \min\left\{ \sum_{h=0}^{H-1}
  \|\phi^t_h\|^2_{(\Sigma^t)^{-1}}, 1  \right\} \leq  2\log \left( \det(\Sigma^{T}) \det(\Sigma^0)^{-1} \right) .
\end{align*}
%and so:
%\begin{align*}
%\Exp\left[ \sum_{t=0}^{T-1} \min\left\{ \sum_{h=0}^{H-1} \|\phi^t_h\|^2_{(\Sigma^t)^{-1}}, 1  \right\} \right] \leq  2 \gamma_T({\lambda}).
%\end{align*}
\label{lem:sum_potential}
\end{lemma}

Also, recall that \algname{}
uses the setting of  $\lambda = \sigma^2 /
\|W^\star\|_2^2$. We will also use that, for $\beta^T$ as defined  in Lemma~\ref{lem:confidence_ball},
\begin{align}
&\beta^T
%\leq 2\sigma^2+ 16\sigma^2\left( d_{\mathcal{X}} \log(5) +  2\log(T) + \log(2)
%+ \log\left( \det(\Sigma^{T})\det(\Sigma^0)^{-1}  \right) \right)\\
=  2\sigma^2+ 8\sigma^2\left( d_{\mathcal{X}} \log(5) +  2\log(T) + \log(4)
 + \log\left( \det(\Sigma^{T})\det(\Sigma^0)^{-1}  \right) \right) \nonumber\\
& \leq 16\sigma^2\left( d_{\mathcal{X}} + \log(T) + \log\left( \det(\Sigma^{T})\det(\Sigma^0)^{-1}  \right)      \right).
\label{eq:beta_T}
\end{align} 
In particular, we can take $C_1=16$ in \algname{}. Also, 
\begin{align}
\Exp[\beta^T]
& \leq 16\sigma^2\left( d_{\mathcal{X}} + \log(T) + \gamma_T(\lambda)      \right).
\label{eq:beta_T_bound}
\end{align} 
using the definition of the information gain.

We now conclude the proof our first main theorem (Theorem~\ref{thm:main1}).

\begin{proof}[Proof of Theorem~\ref{thm:main1}]
Using the per-episode regret bound (Lemma~\ref{lem:regret}), our
confidence ball, failure probability bound (Lemma~\ref{lem:confidence_ball}), and that $V^{t} \leq \Mtwo$,
\begin{eqnarray*}
&&\Exp\left[\textsc{Regret}_{\algname{}}\right] 
=\Exp\left[\sum_{t=0}^{T-1}\left( J^t(x_0;c^t) - J^\star(x_0;c^t)\right) \right]\\
&\leq&\Exp\left[\sum_{t=0}^{T-1}\Exp\left[\mathbf{1}({\mathcal{E}}_{t,cb})\left( J^t(x_0;c^t) - J^\star(x_0;c^t)\right) \mid\mathcal{H}_t\right]\right]
+\sqrt{\Mtwo}\sum_{t=0}^{T-1}\Exp\left[
\mathbf{1}(\overline{\mathcal{E}}_{t,cb})\right]\\
&\leq&\sqrt{H\Mtwo} \sum_{t=0}^{T-1}\Exp\left[
\sqrt{\frac{4\beta^t}{\sigma^2}+ H  } \ \sqrt{ \Exp\left[
\min\left\{ \sum_{h=0}^{H-1} \|\phi_h^t\|^2_{(\Sigma^t)^{-1}} ,1\right\}
\ \Big\vert \ \mathcal{H}_{<t} \right]} \ \right]
+\sqrt{\Mtwo}/2\\
&\leq&\sqrt{H\Mtwo} \sum_{t=0}^{T-1}
\sqrt{\Exp\left[\frac{4\beta^t}{\sigma^2}+ H \right] } \ \sqrt{ \Exp\left[
\min\left\{ \sum_{h=0}^{H-1} \|\phi_h^t\|^2_{(\Sigma^t)^{-1}} ,1\right\}
\right]} 
+\sqrt{\Mtwo}/2\\
&\leq&\sqrt{H\Mtwo} 
\sqrt{\sum_{t=0}^{T-1}\Exp\left[\frac{4\beta^t}{\sigma^2}+ H \right] } \ \sqrt{ \Exp\left[
\sum_{t=0}^{T-1}\min\left\{ \sum_{h=0}^{H-1} \|\phi_h^t\|^2_{(\Sigma^t)^{-1}} ,1\right\}
\right]} 
+\sqrt{\Mtwo}/2\\
&\leq&\sqrt{H\Mtwo} 
\sqrt{T\left(\frac{4 \Exp[\beta^T]}{\sigma^2}+ H \right) } \ \sqrt{ \gamma_T(\lambda)} 
+\sqrt{\Mtwo}/2\\
&\leq&\sqrt{H\Mtwo} 
\sqrt{64 T \big( d_{\mathcal{X}} + \log(T) + \gamma_T(\lambda)+H\big)  } \ \sqrt{ \gamma_T(\lambda)} 
+\sqrt{\Mtwo}/2
\end{eqnarray*} 
where the third
inequality use that $\Exp[ab] \leq\sqrt{ \Exp[a^2]  \Exp[b^2]}$; the
fourth uses the Cauchy-Schwartz inequality; 
the penultimate step uses that $\beta_t$ is non-decreasing, along with
the Lemma~\ref{lem:sum_potential} and the definition of the
information gain; and the final step uses the bound on $\beta^T$ in
Equation~\ref{eq:beta_T_bound}. This completes the proof.
\end{proof}

The proof of our second main theorem (Theorem~\ref{thm:main2}) now follows.

\begin{proof}[Proof of Theorem~\ref{thm:main2}]
By assumption~\ref{asm:variance_cond} on $V^t$ and the per-episode regret
lemma (Lemma~\ref{lem:regret}), 
\begin{align*}
&\mathbf{1}({\mathcal{E}}_{t,cb})  V^t 
\leq \alpha^2 \mathbf{1}({\mathcal{E}}_{t,cb}) J^{t}(x_0; c^t, W^\star)^2 \\
&\leq 2 \alpha^2 J^\star(x_0; c^t, W^\star)^2 + 2 \mathbf{1}({\mathcal{E}}_{t,cb})\alpha^2 \Big(J^t(x_0; c^t, W^\star) - J^\star(x_0; c^t, W^\star)\Big)^2 \\
&\leq 2 \alpha^2 J^\star(x_0; c^t, W^\star)^2 + 
2\alpha^2 H\Mtwo \left(\frac{4\beta^t}{\sigma^2}+ H  \right) \Exp\left[
\min\left\{ \sum_{h=0}^{H-1} \|\phi_h^t\|^2_{(\Sigma^t)^{-1}} ,1\right\}
\ \Big\vert \ \mathcal{H}_{<t} \right]
\end{align*} 

Using this,  and with
Lemma~\ref{lem:regret} and Lemma~\ref{lem:confidence_ball}, we have
\begin{eqnarray*}
&&\Exp\left[\textsc{Regret}_{\algname{}}\right] \\
&\leq&\Exp\left[\sum_{t=0}^{T-1}\Exp\left[\mathbf{1}({\mathcal{E}}_{t,cb})\left( J^t(x_0;c^t) - J^\star(x_0;c^t)\right) \mid\mathcal{H}_t\right]\right]
+\sqrt{\Mtwo}\sum_{t=0}^{T-1}\Exp\left[
\mathbf{1}(\overline{\mathcal{E}}_{t,cb})\right]\\
&\leq& \sum_{t=0}^{T-1}\Exp\left[
\sqrt{H \mathbf{1}({\mathcal{E}}_{t,cb}) V^t\left(\frac{4 \beta^t}{\sigma^2}+ H\right)  } \ \sqrt{ \Exp\left[
\min\left\{ \sum_{h=0}^{H-1} \|\phi_h^t\|^2_{(\Sigma^t)^{-1}} ,1\right\}
\ \Big\vert \ \mathcal{H}_{<t} \right]} \ \right]
+\sqrt{\Mtwo}/2\\
&\leq& \alpha J^*\sqrt{2H }\sum_{t=0}^{T-1}\Exp\left[
\sqrt{\frac{4 \beta^t}{\sigma^2}+ H  } \ \sqrt{ \Exp\left[
\min\left\{ \sum_{h=0}^{H-1} \|\phi_h^t\|^2_{(\Sigma^t)^{-1}} ,1\right\}
\ \Big\vert \ \mathcal{H}_{<t} \right]} \ \right]\\
&+& \alpha\sqrt{2 H^2\Mtwo}
\sum_{t=0}^{T-1}\Exp\left[
\left(\frac{4 \beta^t}{\sigma^2}+ H\right)   \Exp\left[
\min\left\{ \sum_{h=0}^{H-1} \|\phi_h^t\|^2_{(\Sigma^t)^{-1}} ,1\right\}
\ \Big\vert \ \mathcal{H}_{<t} \right]  \right]
+\sqrt{\Mtwo}/2.
\end{eqnarray*} 
where have used that $\sqrt{a+b}\leq\sqrt{a}
+\sqrt{b}$ for positive $a$ and $b$ in the last inequality.

An identical argument to that in the proof of Theorem~\ref{thm:main1}
leads to the first term above being bounded as:
\begin{align*}
&\alpha J^*\sqrt{2H }\sum_{t=0}^{T-1}\Exp\left[
\sqrt{\frac{4 \beta^t}{\sigma^2}+ H  } \ \sqrt{ \Exp\left[
\min\left\{ \sum_{h=0}^{H-1} \|\phi_h^t\|^2_{(\Sigma^t)^{-1}} ,1\right\}
\ \Big\vert \ \mathcal{H}_{<t} \right]} \ \right]\\
& \leq \alpha J^\star \sqrt{128 \gamma_T(\lambda)\left( d_{\mathcal{X}} + \log(T) + \gamma_T(\lambda)+H\right) H T}
\end{align*} 

For the second term, 
\begin{align*}
&  \quad \Exp\left[\sum_{t=0}^{T-1}
\left(\frac{4 \beta^t}{\sigma^2}+ H\right)   \Exp\left[
\min\left\{ \sum_{h=0}^{H-1} \|\phi_h^t\|^2_{(\Sigma^t)^{-1}} ,1\right\}
\ \Big\vert \ \mathcal{H}_{<t} \right]  \right]\\
&=\Exp\left[\sum_{t=0}^{T-1}
\left(\frac{4 \beta^t}{\sigma^2}+ H\right)   
\min\left\{ \sum_{h=0}^{H-1} \|\phi_h^t\|^2_{(\Sigma^t)^{-1}} ,1\right\}\right]\\
&\leq\Exp\left[\left(\frac{4 \beta^T}{\sigma^2}+ H\right)   \sum_{t=0}^{T-1}
\min\left\{ \sum_{h=0}^{H-1} \|\phi_h^t\|^2_{(\Sigma^t)^{-1}} ,1\right\}\right]\\
&\leq 2 \Exp\left[\left(\frac{4 \beta^T}{\sigma^2}+ H\right)    \log\left( \det(\Sigma^{T})\det(\Sigma^0)^{-1}  \right)\right]\\
&\leq 128 \Exp\left[
\big( d_{\mathcal{X}} + \log(T) + \log\left( \det(\Sigma^{T})\det(\Sigma^0)^{-1}  \right) +H   \big)
\log\left( \det(\Sigma^{T})\det(\Sigma^0)^{-1}  \right)\right]\\
& \leq 128 \left( H + d_{\mathcal{X}} + \log(T)   \right) \gamma_T(\lambda) +  128 \gamma_{2,T}(\lambda)
\end{align*}
where we have used that $\beta^t$ is measurable with respect
to $\mathcal{H}_{<t}$ in the first equality;
that $\beta^t$ is non-decreasing in the first
inequality;   Lemma~\ref{lem:sum_potential} in the second
inequality; our bound on $\beta^T$ in Equation~\ref{eq:beta_T} in the
third inequality; 
and the definition of $\gamma_T(\lambda)$ and $ \gamma_{2,T}(\lambda)$
in the final step.

The proof is completed via substitution.
\end{proof}

\subsection{Confidence Bound and Potential Function Analysis} \label{section:confidence}

\begin{proof}[Proof of Lemma~\ref{lem:confidence_ball}]
The center of the confidence ball, $\overline{W}^t$, is the
minimizer of the ridge regression objective in
Equation~\ref {eq:regression}; its closed-form expression is:
\begin{align*}
\overline{W}^t := \sum_{\tau = 0}^{t-1}\sum_{h=0}^{H-1} x^\tau_{h+1} (\phi^\tau_h)^{\top} (\Sigma^t)^{-1},
\end{align*}  
where 
$\Sigma^t = \lambda I +
\sum_{\tau = 0}^{t-1}\sum_{h=0}^{H-1} \phi^\tau_h (\phi^\tau_h)^{\top}
$.
Using that $x^\tau_{h+1} = W^\star\phi^\tau_h + \epsilon^\tau_h$ with
$\epsilon^\tau_h\sim\mathcal{N}(0, \sigma^2 I)$, 
\begin{align*}
&\overline{W}^t - W^\star = \sum_{\tau=0}^{t-1} \sum_{h=0}^{H-1}
  x^\tau_{h+1}(\phi^\tau_h)^{\top} (\Sigma^t)^{-1}- W^\star \\
& = \sum_{\tau=0}^{t-1} \sum_{h=0}^{H-1} (W^\star\phi^\tau_h + \epsilon^\tau_h ) (\phi^\tau_h)^{\top} (\Sigma^t)^{-1} - W^\star \\
& =  W^\star \left(\sum_{\tau=0}^{t-1}\sum_{h=0}^{H-1}  \phi^\tau_h(\phi^\tau_h)^{\top}\right) (\Sigma^t)^{-1} - W^\star +  \sum_{\tau=0}^{t-1} \sum_{h=0}^{H-1}  \epsilon^\tau_h (\phi^\tau_h)^{\top} (\Sigma^t)^{-1}\\
& = - \lambda W^\star \left(\Sigma^t\right)^{-1} + \sum_{\tau=0}^{t-1} \sum_{h=0}^{H-1}  \epsilon^\tau_h (\phi^\tau_h)^{\top} (\Sigma^t)^{-1}.
\end{align*}

%Denote $\Delta_{W} = \overline{W}^t - W^\star$.
%Now let us look at $\left\| \Delta_{W} \left(\Sigma^t\right)^{1/2} \right\|_2$. 
For any $0<\delta_t<1$, using Lemma~\ref{lemma:self_norm_matrix}, it holds with probability at least $1-\delta_t$,
\begin{align*}
& \left\| \left(\overline{W}^t - W^\star\right)\left(\Sigma^t\right)^{1/2}    \right\|_2 
\leq \left\| \lambda W^\star \left(\Sigma^t\right)^{-1/2} \right\|_2 
+ \left\| \sum_{\tau=0}^{t-1}\sum_{h=0}^{H-1}  \epsilon^\tau_h (\phi^\tau_h)^{\top} (\Sigma^t)^{-1/2} \right\|_2 \\
& \leq \sqrt{\lambda}\|W^\star\|_2 + \sigma \sqrt{ 8d_{\mathcal{X}} \log(5) + 8\log\left(\det(\Sigma^t)\det(\Sigma^0)^{-1} / \delta_t \right) }.
\end{align*} 
where we have also used the triangle inequality. Therefore,
$\mathrm{Pr}( \overline{\mathcal{E}}_{t,cb}) \leq \delta_t$.

We seek to bound $\sum_{t=0}^{\infty} \mathrm{Pr}( \overline{\mathcal{E}}_{t,cb})$.
Due to that at $t=0$ we have initialized $\textsc{Ball}^{0}$ to
contain $W^\star$, we have $\mathrm{Pr}( \overline{\mathcal{E}}_{0,cb})=0$.
For $t\geq 1$, let us assign failure probability $\delta_t=(3/\pi^2)/t^2$ for the
$t$-th event, which, using the above, gives us an upper bound on the sum failure
probability as $\sum_{t=1}^{\infty} \mathrm{Pr}(
  \overline{\mathcal{E}}_{t,cb}) <\sum_{t=1}^{\infty} (1/t^2)
(3/\pi^2) = 1/2$. This completes the proof. 
\end{proof}

\begin{proof}[Proof of Lemma~\ref{lem:sum_potential}] 
Recall that $\Sigma^{t+1} = \Sigma^t + \sum_{h=0}^{H-1} \phi^t_h\left(\phi^t_h\right)^{\top}$ and $\Sigma^0 = \lambda I $. 
First use $x \leq 2\log(1+x)$ for $x\in [0,1]$, we have:
\begin{align*}
&\min\left\{   \sum_{h=0}^{H-1} \|\phi_h^t\|^2_{(\Sigma^t)^{-1}}, 1 \right\} \leq 2\log\left(1 +  \sum_{h=0}^{H-1} \|\phi_h^t\|^2_{(\Sigma^t)^{-1}}\right).
\end{align*} For $\Sigma^{t+1}$, using its recursive formulation, we have:
\begin{align*}
&\log\det\left(\Sigma^{t+1}\right) = \log\det\left(\Sigma^t \right) + \log\det\left( I + \left(\Sigma^t\right)^{-1/2} \sum_{h=0}^{H-1}\phi^t_h(\phi^t_h)^{\top}\left(\Sigma^t\right)^{-1/2}\right)
\end{align*} 

Denote the eigenvalues of $\left(\Sigma^t\right)^{-1/2} \sum_{h=0}^{H-1}\phi^t_h(\phi^t_h)^{\top}\left(\Sigma^t\right)^{-1/2}$ as $\sigma_i$ for $i \geq 1$.  
We have 
\begin{align*}
\log\det\left( I + \left(\Sigma^t\right)^{-1/2} \sum_{h=0}^{H-1}\phi^t_h(\phi^t_h)^{\top}\left(\Sigma^t\right)^{-1/2}\right) = \log\prod_{i\geq 1} \left(1 + \sigma_i\right) \geq \log\left(1 + \sum_{i\geq 1} \sigma_i \right),
\end{align*} where the last inequality uses that $\sigma_i \geq
0$ for all $i$. Using the above and the definition of the trace, 
\begin{align*}
&\log\det\left( I + \left(\Sigma^t\right)^{-1/2} \sum_{h=0}^{H-1}\phi^t_h(\phi^t_h)^{\top}\left(\Sigma^t\right)^{-1/2}\right)  \geq \log\left( 1 +  \tr\left(\left(\Sigma^t\right)^{-1/2} \sum_{h=0}^{H-1}\phi^t_h(\phi^t_h)^{\top}\left(\Sigma^t\right)^{-1/2}\right) \right) \\
&  = \log\left(1 +  \sum_{h=0}^H (\phi_h^t)^{\top} (\Sigma^t)^{-1} \phi_h^t \right)
\end{align*}
By telescoping the sum, 
\begin{align*}
&2\sum_{t=0}^{T-1}  \log\left(1 +  \sum_{h=0}^H (\phi_h^t)^{\top} (\Sigma^t)^{-1} \phi_h^t \right) 
\leq 2\sum_{t=1}^{T-1} \left( \log\det\left(\Sigma^{t+1} \right) - \log\det\left(\Sigma^t \right)\right)\\
&= \log \left( \det(\Sigma^{T}) \det(\Sigma^0)^{-1} \right) ,
\end{align*} 
which completes the proof.
\end{proof}

\section{Technical Lemmas}\label{section:technical}

\begin{lemma}[Chi Squared Distance Between Two Gaussians] For 
  Gaussian distributions $\mathcal{N}(\mu_1, \sigma^2 I)$ and
  $\mathcal{N}(\mu_2, \sigma^2 I)$, the (squared) chi-squared distance
  between $\mathcal{N}_1$ and $\mathcal{N}_2$ is:
\begin{align*}
\int \frac{(\mathcal{N}_1(z)-\mathcal{N}_2(z))^2}{\mathcal{N}_1(z)} dz
= 
\exp\left( \frac{\|\mu_1-\mu_2\|^2}{2\sigma^2}\right) -1
\end{align*} 
\label{lem:chi_squared_gaussian}
\end{lemma}
\begin{proof}
Observe that:
\begin{align*}
\int \frac{(\mathcal{N}_1(z)-\mathcal{N}_2(z))^2}{\mathcal{N}_1(z)} dz
  =  \int \mathcal{N}_1(z) - 2\mathcal{N}_2(z) +
  \frac{\mathcal{N}_2(z)^2}{\mathcal{N}_1(z)} dz 
  = -1 + \int\frac{\mathcal{N}_2(z)^2}{\mathcal{N}_1(z) } dz. 
\end{align*}
Note that for $\mathcal{N}_2^2(z) / \mathcal{N}_1(z)$, we have:
\begin{align*}
\mathcal{N}_2^2(z) / \mathcal{N}_1(z) = \frac{1}{Z} \exp\left( -\frac{1}{2\sigma^2}\left( 2\| z - \mu_2\|_2^2 - \| z - \mu_1\|_2^2 \right) \right),
\end{align*} where $Z$ is the normalization constant for
$\mathcal{N}(0,\sigma^2 I)$, i.e. $Z = \int \exp\left(
  -\frac{1}{2\sigma^2} \|z\|_2^2 \right) dz$.

For $2\| z - \mu_2\|_2^2 - \| z - \mu_1\|_2^2$, we can verify that:
\begin{align*}
2\| z - \mu_2\|_2^2 - \| z - \mu_1\|_2^2 = \| z + (\mu_1 - 2\mu_2)\|_2^2 -2\|\mu_1 - \mu_2\|_2^2.
\end{align*}This implies that:
\begin{align*}
&\int\frac{\mathcal{N}_2(z)^2}{\mathcal{N}_1(z) } dz = \frac{1}{Z} 
\int \exp\left(  -\frac{1}{2\sigma^2} \left( \| z - (2\mu_2 - \mu_1)\|_2^2  - 2\|\mu_1 - \mu_2 \|\right)   \right) dz\\
& = \frac{1}{Z} \exp\left(  \frac{\|\mu_1- \mu_2\|_2^2}{\sigma^2} \right) \int\exp\left( -\frac{1}{2\sigma^2} \| z - (2\mu_2 - \mu_1)\|_2^2 \right) dz\\
& = \exp\left( \frac{\|\mu_1- \mu_2\|_2^2}{\sigma^2} \right),
\end{align*}  which concludes the proof.
\end{proof}

\begin{lemma}[Expectation Difference Under Two
  Gaussians] \label{lem:mean_difference} 
For   Gaussian distribution $\mathcal{N}(\mu_1, \sigma^2 I)$ and
  $\mathcal{N}(\mu_2, \sigma^2 I)$, and for any (appropriately
  measurable) positive function $g$, it holds that:
\begin{align*}
\Exp_{z\sim \mathcal{N}_1} [g(z)] -\Exp_{z\sim \mathcal{N}_2} [g(z)] 
\leq \min\left\{\frac{\|\mu_1-\mu_2\|}{\sigma} ,1\right\}
\ \sqrt{\Exp_{z\sim \mathcal{N}_1}\left[  g(z) ^2\right]}.
\end{align*} 
\end{lemma}

\begin{proof}
Define $m_i = \Exp_{z\sim \mathcal{N}_i} [g(z)]$
for $i\in \{0,1\}$. We have:
\begin{align*}
m_1-m_2
&= \Exp_{z\sim \mathcal{N}_1} [g(z) (1-\frac{\mathcal{N}_2(z)}{\mathcal{N}_1(z)})]\\
&\leq \sqrt{\Exp_{z\sim \mathcal{N}_1} [g(z)^2]}
\sqrt{\int \frac{(\mathcal{N}_1(z)-\mathcal{N}_2(z))^2}{\mathcal{N}_1(z)} dz}\\
&= \sqrt{\Exp_{z\sim \mathcal{N}_1} [g(z)^2]}
\sqrt{\exp\left( \frac{\|\mu_1-\mu_2\|^2}{2\sigma^2}\right) -1}
\end{align*} 
where we have used the previous chi-squared distance bound.
Also since $m_2$ is positive, 
\begin{align*}
m_1-m_2\leq m_1 \leq \sqrt{\Exp_{z\sim \mathcal{N}_1} [g(z)^2]},
\end{align*} 
and so
\begin{align*}
m_1-m_2
\leq \sqrt{\Exp_{z\sim \mathcal{N}_1} [g(z)^2]}
\sqrt{\min\left\{
\exp\left( \frac{\|\mu_1-\mu_2\|^2}{2\sigma^2}\right) -1
, 1\right\}}
\end{align*} 
Now if the $\min$ is not achieved by $1$, then
$\frac{\|\mu_1-\mu_2\|^2}{2\sigma^2} \leq 1$. And since $\exp(x) \leq 1+2x$
for $0\leq x \leq 1$, we have:
\begin{align*}
\min\left\{
\exp\left( \frac{\|\mu_1-\mu_2\|^2}{2\sigma^2}\right) -1
, 1\right\}
\leq
\min\left\{
1+ \frac{\|\mu_1-\mu_2\|^2}{\sigma^2} -1
, 1\right\} 
=
\min\left\{
\frac{\|\mu_1-\mu_2\|^2}{\sigma^2} , 1\right\} .
\end{align*} 
which completes the proof.
\end{proof}

\begin{lemma}[Self-Normalized Bound for Vector-Valued Martingales; \citep{abbasi2011improved}] 
Let $\{\varepsilon_i\}_{i=1}^{\infty}$ be a real-valued stochastic
process with corresponding filtration $\{\calF_{i}\}_{i=1}^{\infty}$
such that $\varepsilon_i$ is $\calF_i$ measurable, $\Exp[
\varepsilon_i | \calF_{i-1} ] = 0$, and $\varepsilon_i$ is
conditionally $\sigma$-sub-Gaussian with $\sigma\in\mathbb{R}^+$. Let
$\{X_i\}_{i=1}^{\infty}$ be a stochastic process with $X_i\in
\mathcal{H}$ (some Hilbert space) and $X_i$ being $\calF_t$ measurable. Assume that a
linear operator $V:\mathcal{H}\to\mathcal{H}$ is positive definite,
i.e., $x^{\top} V x > 0$ for any $x\in\mathcal{H}$. For any $t$,
define the linear operator $V_t = V + \sum_{i=1}^{t} X_iX_i^{\top}$ (here
$xx^{\top}$ denotes outer-product in $\mathcal{H}$). With probability
at least $1-\delta$, we have for all $t\geq 1$: 
\begin{align*}
\left\|  \sum_{i=1}^{t} X_i \varepsilon_i \right\|^2_{V_t^{-1}} \leq 2\sigma^2 \log\left(\frac{ \det( V_t)^{1/2} \det(V)^{-1/2} }{ \delta  }\right).
\end{align*}
\label{lemma:self_normalized}
\end{lemma}
%Note that this lemma is dimension-independent, i.e., it holds for $X\in\mathbb{R}^{\infty}$. 

\iffalse
\begin{lemma}[Self-Normalized Bound for Matrix-Valued Martingales]
  Consider a stochastic process $\{\phi(x_i, u_i),
  x'_{i}\}_{i=1}^\infty$ with $x'_i = W \phi(x_i,a_i) + \epsilon_i$
  where $\epsilon_i\sim \mathcal{N}\left(0, \sigma^2 I \right)$,
  $x\in\mathbb{R}^{d}$.  Denote $V_N = \sum_{i=1}^N \phi(x_i,
  u_i)\phi(x_i,u_i)^{\top} + \lambda I$ with
  $\lambda\in\mathbb{R}^+$. Then, with probability at least
  $1-\delta$, for any $N$,  we have: 
\begin{align*}
 \left\| \sum_{i=1}^N \epsilon_i \phi_i^{\top} V_N^{-1/2} \right\|_2^2 \leq 8\sigma^2 d\log\left(5\right) + 8\sigma^2\log\left( \frac{\det(V_N)^{1/2}\det(\lambda I)^{-1/2}}{\delta} \right) 
\end{align*}
\label{lemma:self_norm_matrix}
\end{lemma}
\fi

We generalize this lemma as follows:
\begin{lemma}[Self-Normalized Bound for Matrix-Valued Martingales]
Let $\{\varepsilon_i\}_{i=1}^{\infty}$ be a $d$-dimensional vector-valued stochastic
process with corresponding filtration $\{\calF_{i}\}_{i=1}^{\infty}$
such that $\varepsilon_i$ is $\calF_i$ measurable, $\Exp[
\varepsilon_i | \calF_{i-1} ] = 0$, and $\varepsilon_i$ is
conditionally $\sigma$-sub-Gaussian with
$\sigma\in\mathbb{R}^+$.\footnote{We say a vector-valued, random variable $z$
  is $\sigma$-sub-Gaussian if $w\cdot z$ is $\sigma$-sub-Gaussian for
  every unit vector $w$.}
Let
$\{X_i\}_{i=1}^{\infty}$ be a stochastic process with $X_i\in
\mathcal{H}$ (some Hilbert space) and $X_i$ being $\calF_t$ measurable. Assume that a
linear operator $V:\mathcal{H}\to\mathcal{H}$ is positive definite. For any $t$,
define the linear operator $V_t = V + \sum_{i=1}^{t} X_iX_i^{\top}$
Then, with probability at least
  $1-\delta$, we have for all $t$,  we have: 
\begin{align*}
 \left\| \sum_{i=1}^t \epsilon_i X_i^{\top} V_t^{-1/2} \right\|_2^2 
\leq 8\sigma^2 d\log\left(5\right) + 8\sigma^2\log\left( \frac{\det(V_t)^{1/2}\det(V)^{-1/2}}{\delta} \right) 
\end{align*}
\label{lemma:self_norm_matrix}
\end{lemma}

\begin{proof}
Denote $S = \sum_{i=1}^t \epsilon_i X_i^{\top} $.  Let us
form an $\epsilon$-net, in $\ell_2$ distance,  $\mathcal{C}$ over the unit ball $\{w: \|w\|_2 \leq 1,
w\in\mathbb{R}^{d} \}$. Via a standard covering argument
(e.g. \citep{shalev2014understanding}), we can choose $\mathcal{C}$
such that
$\log\left(| \mathcal{C} |\right) \leq d\log(1+ 2/\epsilon)$.     

Consider a fixed $w \in \mathcal{C}$ and $w^{\top} S = \sum_{i=1}^t
w^{\top}\epsilon_i X_i^T $. Note that $w^{\top}\epsilon_i$
%$\sim \mathcal{N}(0, \sigma^2 \|w\|_2^2)$ which 
is a $\sigma$-sub
Gaussian due to $\|w\|_2 \leq 1$.   Hence,
Lemma~\ref{lemma:self_normalized} implies that with probability
at least $1-\delta$,   for all $t$,
\begin{align*}
\left\| V_t^{-1/2} \sum_{i=1}^t X_i \left(w^{\top}\epsilon_i \right)  \right\|_2 \leq \sqrt{2}\sigma \sqrt{ \log\left(\frac{ \det( V_t)^{1/2} \det(V )^{-1/2} }{ \delta  }\right) }. 
\end{align*}
Now apply a union bound over $\mathcal{C}$, we get that with probability at least $1-\delta$:
\begin{align*}
\forall w\in\mathcal{C}: \left\| V_t^{-1/2} \sum_{i=1}^t X_i \left(w^{\top}\epsilon_i \right)  \right\|_2 \leq %\sqrt{2}\sigma \sqrt{d \log\left((1+2/\epsilon)\frac{\det( V_t)^{1/2} \det(V)^{-1/2} }{ \delta  }\right) }. 
\sqrt{2}\sigma \sqrt{d\log\left(1+2/\epsilon\right) + \log\left( \frac{\det(V_t)^{1/2}\det(V)^{-1/2}}{\delta} \right) }.
\end{align*}
For any $w$ with $\|w\|_2 \leq 1$, there exists a $w'\in\mathcal{C}$
such that $\|w - w'\|_2 \leq \epsilon$. Hence, for all $w$ such that $\|w\|_2\leq 1$,
\begin{align*}
 \left\| V_t^{-1/2} \sum_{i=1}^t X_i \left(w^{\top}\epsilon_i \right)  \right\|_2 
&\leq \sqrt{2}\sigma \sqrt{d\log\left(1+2/\epsilon\right) + \log\left( \frac{\det(V_t)^{1/2}\det(V )^{-1/2}}{\delta} \right) } \\
& \quad + \epsilon \left\| \sum_{i=1}^t \epsilon_i X_i^{\top} V_t^{-1/2} \right\|_2.
\end{align*} 
By the definition of the spectral norm, this implies that:
\begin{align*}
 \left\| \sum_{i=1}^t \epsilon_i X_i^{\top} V_t^{-1/2} \right\|_2 \leq \frac{1}{1-\epsilon} \sqrt{2}\sigma \sqrt{d\log\left(1+2/\epsilon\right) + \log\left( \frac{\det(V_t)^{1/2}\det(V)^{-1/2}}{\delta} \right) }
\end{align*} 
Taking $\epsilon = 1/2$ concludes the proof. 
\end{proof}

\begin{lemma}\label{lem:finite_info_gain}
For any sequence $x_0, \ldots x_{T-1}$ such that, for $t<T$, $x_t\in \mathbb{R}^d$ and $\|x_t\|_2 \leq B\in\mathbb{R}^+$, we have:
\begin{align*}
\log\det\left( I + \frac{1}{\lambda} \sum_{t=0}^{T-1}x_t x_t^{\top}\right)
\leq  d\log\left( 1 + \frac{TB^2}{ d \lambda } \right).
\end{align*}
\end{lemma}
\begin{proof}
Denote the eigenvalues of $ \sum_{t=0}^{T-1}x_t x_t^{\top}$ as
$\sigma_1, \dots \sigma_d$, and note:
\begin{align*}
\sum_{i=1}^d \sigma_i = \tr\left( \sum_{t=0}^{T-1}x_t x_t^{\top}  \right) \leq T B^2.
\end{align*} 
Using the AM-GM inequality,
\begin{align*}
&\log\det\left( I + \frac{1}{\lambda} \sum_{t=0}^{T-1}x_t x_t^{\top}\right) = \log\left( \prod_{i=1}^d \left(1 + \sigma_i / \lambda \right) \right) \\
&= d\log\left( \prod_{i=1}^d \left(1 + \sigma_i / \lambda \right)
  \right)^{1/d} 
\leq d\log\left( \frac{1}{d}\sum_{i=1}^d \left(1 + \sigma_i / \lambda \right) \right) 
\leq d \log\left( 1 + \frac{ TB^2  }{d\lambda}  \right),
\end{align*} 
which concludes the proof. 
\end{proof}

\iffalse
\begin{lemma}\label{lem:finite_info_gain}
When $\phi\in \mathbb{R}^d$ and $\|\phi\|_2 \leq B\in\mathbb{R}^+$, we have:
\begin{align*}
\lambda(T) \leq  d\log\left( 1 + \frac{THB^2}{ d \lambda } \right).
\end{align*}
\end{lemma}
\begin{proof}
Consider $\log\det\left( I + \frac{1}{\lambda} \sum_{t=0}^{T-1}\sum_{h=0}^{H-1} \phi_h^t (\phi_h^t)^{\top}\right)$. 

Let us denote the eigenvalue of $ \sum_{t=0}^{T-1}\sum_{h=0}^{H-1} \phi_h^t (\phi_h^t)^{\top}$ as $\sigma_1, \dots \sigma_d$. Note that:
\begin{align*}
\sum_{i=1}^d \sigma_i = \tr\left( \sum_{t=0}^{T-1}\sum_{h=0}^{H-1} \phi_h^t (\phi_h^t)^{\top}  \right) \leq TH B^2.
\end{align*} Hence,
\begin{align*}
&\log\det\left( I + \frac{1}{\lambda} \sum_{t=0}^{T-1}\sum_{h=0}^{H-1} \phi_h^t (\phi_h^t)^{\top}\right) = \log\left( \prod_{i=1}^d \left(1 + \sigma_i / \lambda \right) \right) \\
&= \sum_{i=1}^d \log\left( 1 + \frac{\sigma_i}{\lambda} \right) \leq \sum_{i=1}^d \log\left( 1 + \frac{ THB^2  }{d\lambda}  \right) = d \log\left( 1 + \frac{ THB^2  }{d\lambda}  \right).
\end{align*} Note that the right hand side of the above inequality is deterministic, which concludes the proof. 
\end{proof}
\fi

\section{Simulation Setups and Results}
\label{sec:appsim}

Below, we provide simulation setups, including the details of environments and parameter settings.
Specifically, the hyper-parameters, namely, 1) variance of random control variation for MPPI, 2) temperature parameter for MPPI, 3) planning horizon, 4) number of planning samples, 5) prior parameter $\lambda$, 6) posterior reshaping constant, 7) number of episodes between model updates, 8) number of features, 9) RFF bandwidth, are presented.

Note parameters were tuned in the following way: we first tuned MPPI parameters on ground truth models, then we tuned number of RFFs, their bandwidth, prior parameter, and posterior reshaping constant.

\subsection{Gym Environments}
The hyper-parameters used for InvertedPendulum, Acrobot, CartPole, Mountain Car, Reacher, and Hopper are shown in Table \ref{tab:paramip}, \ref{tab:paramab}, \ref{tab:paramcp}, \ref{tab:parammc}, \ref{tab:paramrc}, and \ref{tab:paramhp}, respectively.
We used \texttt{JULIA\_NUM\_THREADS=12} for all the Gym experiments.

We mention that we tested many heuristics to improve performance such as input normalization, different prior parameter for each output dimension, using multiple bandwidth of RFFs, ensemble of RFF models, warm start of planner, experience replay, etc., however, none of them {\em consistently} improved the performance across tasks.  Therefore we present the results with no such heuristics in this paper.
Interestingly, increasing number of RFFs for some contact-rich dynamics such as Hopper did not reduce the modeling error significantly.
Being able to model some of the critical interactions such as contacts seems to be the key for the success of such a complicated environment.

\begin{table}
	\caption{Hyper-parameters used for InvertedPendulum environment.}
	\label{tab:paramip}
	\centering
	\begin{tabular}{l|c||l|c}
		\toprule
		MPPI Hyper-parameters & Value & \algname{} Hyper-parameters & Value\\
		\midrule
		variance of controls     & $0.2^2$ & number of features & $200$ \\
		temperature parameter      & $0.1$ & RFF bandwidth      & $5.5$  \\
		planning horizon           & $10$  & prior parameter    & $10^{-4}$ \\
		number of planning samples & $256$ & posterior reshaping constant & $0$ \\
		                           &       & episodes between model updates & $1$ \\
		\bottomrule
	\end{tabular}
\end{table}

\begin{table}
	\caption{Hyper-parameters used for Acrobot environment.}
	\label{tab:paramab}
	\centering
	%\begin{scriptsize}
	\begin{tabular}{l|c||l|c}
		\toprule
		MPPI Hyper-parameters & Value & \algname{} Hyper-parameters & Value\\
		\midrule
		variance of controls     & $0.2^2$ & number of features & $200$ \\
		temperature parameter      & $0.3$ & RFF bandwidth      & $4.5$  \\
		planning horizon           & $30$  & prior parameter    & $0.01$ \\
		number of planning samples & $256$ & posterior reshaping constant & $10^{-3}$ \\
		                           &       & episodes between model updates & $1$ \\
		\bottomrule
	\end{tabular}
	%\end{scriptsize}
\end{table}

\begin{table}
	\caption{Hyper-parameters used for CartPole environment.}
	\label{tab:paramcp}
	\centering
	%\begin{scriptsize}
	\begin{tabular}{l|c||l|c}
		\toprule
		MPPI Hyper-parameters & Value & \algname{} Hyper-parameters & Value\\
		\midrule
		variance of controls     & $0.2^2$ & number of features & $200$ \\
		temperature parameter      & $0.1$ & RFF bandwidth      & $1.5$  \\
		planning horizon           & $50$  & prior parameter    & $5\times10^{-4}$ \\
		number of planning samples & $128$ & posterior reshaping constant & $10^{-4}$ \\
		                           &       & episodes between model updates & $1$ \\
		\bottomrule
	\end{tabular}
\end{table}

\begin{table}
	\caption{Hyper-parameters used for Mountain Car environment.}
	\label{tab:parammc}
	\centering
	\begin{tabular}{l|c||l|c}
		\toprule
		MPPI Hyper-parameters & Value & \algname{} Hyper-parameters & Value\\
		\midrule
		variance of controls     & $0.3^2$ & number of features & $100$ \\
		temperature parameter      & $0.2$ & RFF bandwidth      & $1.3$  \\
		planning horizon           & $110$ & prior parameter    & $0.01$ \\
		number of planning samples & $512$ & posterior reshaping constant & $10^{-6}$ \\
		                           &       & episodes between model updates & $1$ \\
		\bottomrule
	\end{tabular}
\end{table}
 
\begin{table}
	\caption{Hyper-parameters used for Reacher environment.}
	\label{tab:paramrc}
	\centering
	%\begin{scriptsize}
	\begin{tabular}{l|c||l|c}
		\toprule
		MPPI Hyper-parameters & Value & \algname{} Hyper-parameters & Value\\
		\midrule
		variance of controls     & $0.2^2$ & number of features & $300$ \\
		temperature parameter      & $0.3$ & RFF bandwidth      & $4.0$  \\
		planning horizon           & $20$  & prior parameter    & $0.01$ \\
		number of planning samples & $256$ & posterior reshaping constant & $0$ \\
		                           &       & episodes between model updates & $4$ \\
		\bottomrule
	\end{tabular}
	%\end{scriptsize}
\end{table}

\begin{table}
	\caption{Hyper-parameters used for Hopper environment.}
	\label{tab:paramhp}
	\centering
	%\begin{scriptsize}
	\begin{tabular}{l|c||l|c}
		\toprule
		MPPI Hyper-parameters & Value & \algname{} Hyper-parameters & Value\\
		\midrule
		variance of controls     & $0.2^2$ & number of features & $200$ \\
		temperature parameter      & $0.2$ & RFF bandwidth      & $12.0$  \\
		planning horizon           & $128$ & prior parameter    & $0.005$ \\
		number of planning samples & $64$ & posterior reshaping constant & $0.01$ \\
		                           &       & episodes between model updates & $1$ \\
		\bottomrule
	\end{tabular}
	%\end{scriptsize}
\end{table}

\subsection{Maze}
In the Maze environment, states and controls are continuous and the agent plans over continuous spaces; however, the dynamics is given by 1) $x_{h+1}=x_{h}+[-0.5,0]^{\top}$ (i.e., moving one step left) if $\ceil{2u_h}=-1$, 2) $x_{h+1}=x_{h}+[0,-0.5]^{\top}$ (i.e., moving one step up) if $\ceil{2u_h}=0$, 3) $x_{h+1}=x_{h}+[0.5,0]^{\top}$ (i.e., moving one step right) if $\ceil{2u_h}=1$, and 4) $x_{h+1}=x_{h}+[0,0.5]^{\top}$ (i.e., moving one step down) if $\ceil{2u_h}=2$, except for the case there is a wall in the direction of travel, which then ends up $x_{h+1}=x_{h}$.
%State $x_h$ in the environment takes value in $[-1,1]^2\subset\R^2$ and the initial state $x_0$ is $[-1, -1]^{\top}$.
%Control $u$ in the environment takes value in $[-1,1]\subset\R$, and the aim is to bring an agent to the goal state $[1,1]^{\top}$.  Although states and controls are continuous and the agent plans over continuous spaces, the dynamics is given by 1) $x_{h+1}=x_{h}+[-0.5,0]^{\top}$ (i.e., moving one step left) if $\ceil{2u_h}=-1$, 2) $x_{h+1}=x_{h}+[0,-0.5]^{\top}$ (i.e., moving one step up) if $\ceil{2u_h}=0$, 3) $x_{h+1}=x_{h}+[0.5,0]^{\top}$ (i.e., moving one step right) if $\ceil{2u_h}=1$, and 4) $x_{h+1}=x_{h}+[0,0.5]^{\top}$ (i.e., moving one step down) if $\ceil{2u_h}=2$, except for the case there is a wall in the direction of travel, which then ends up $x_{h+1}=x_{h}$.
%
%The negative cost (reward) used in this environment is $- c(x_h,u_h) = 8 - \|x_{h}-[1,1]^{\top}\|_2^2$.
%For Maze experiment, we tested several covariance scales, i.e., the posterior reshaping constant of Thompson sampling, in addition to random walk that takes action uniformly sampled within $[-1,1]\in\R$ and PETS-CEM that is a representative model-based RL which uses uncertainty of dynamics but without exploration.
%Across four random seeds, the success rate (i.e., the number of runs in which the agent reaches the goal within $50$ episodes per the total number of runs) of the best setting of \algname{} was $1.0$, while those for random walk and PETS-CEM were $0.0$.  For the best setting of \algname{}, the average number of episodes required for the first success was $25.0$.

The hyper-parameters of Maze experiments are shown in Table \ref{tab:parammaze}.  Note the number of features is $100$ because one hot vector (e.g., $\phi(x,u)=[1,0,\ldots,0]^{\top}$ if $x\leq-0.75$ and $u\leq-0.5$) in this maze environment is $100$ dimension.
Table \ref{tab:parammaze} also includes the parameters used for PETS-CEM; we used the recommended values as in the paper and the codebase, except for the planning horizon.
The planning horizon was set to be the same as the MPPI counterpart.
We used \texttt{JULIA\_NUM\_THREADS=12} for all the Maze experiments.

\begin{table}
	\caption{Hyper-parameters used for Maze environment.}
	\label{tab:parammaze}
	\centering
	%\begin{scriptsize}
	\begin{tabular}{l|c||l|c}
		\toprule
		Planner Hyper-parameters & Value & \algname{} Hyper-parameters & Value\\
		\midrule
		variance of controls     & $0.3^2$ & number of features & $100$ \\
		temperature parameter     & $0.05$ & prior parameter    & $0.01$ \\
		MPPI planning horizon       & $50$ & posterior reshaping constant & $10^{-3}$ (\textit{best}) \\
		MPPI planning samples     & $1024$ & episodes between model updates & $1$ \\
		PETS-CEM horizon            & $50$ & & \\
		PETS-CEM samples           & $500$ & & \\
		PETS-CEM elite size         & $50$ & & \\
		\bottomrule
	\end{tabular}
	%\end{scriptsize}
\end{table}

\subsection{Armhand with Model Ensemble Features}

%This task involves a 33 DOF anthropomorphic hand at the end of a robot arm attempting to grasp and lift an object to a desired target position. We take the perspective that most model parameters of a robotic system will be known, such as kinematic lengths, actuator specifications, and inertial configurations. Since we would like robots to operate in the wild, some dynamical properties may be unknown: in this case, the manipulated object's dynamical properties. Said another way, the robot knows about itself, but only a little about the object.

In table \ref{tab:paramarmhand}, we list the dynamical properties that were randomized to make our ensemble. We use uniform distributions to present a window of possible, realistic values for the parameters: for example, we randomize the objects mass between 0.1 and 1.0 kg. The center of mass distributions is the deviation from the center of the sphere, while the moments of inertia parameter is one value applied to all elements of a diagonal inertia matrix for the object. The contact parameters are specific to the MuJoCo dynamics simulator we use \cite{todorov2012mujoco}, and are the parameters of internal contact model of the simulation. The range of values of the parameters allow for objects in the ensemble to have different softness and rebound effects.

Also, table \ref{tab:prederrarmhand} lists learned model predictive error for different features, indicating that the ensemble of MuJoCo model successfully captured the true dynamics.

\begin{figure}
  \begin{center}
  \includegraphics[width=\textwidth]{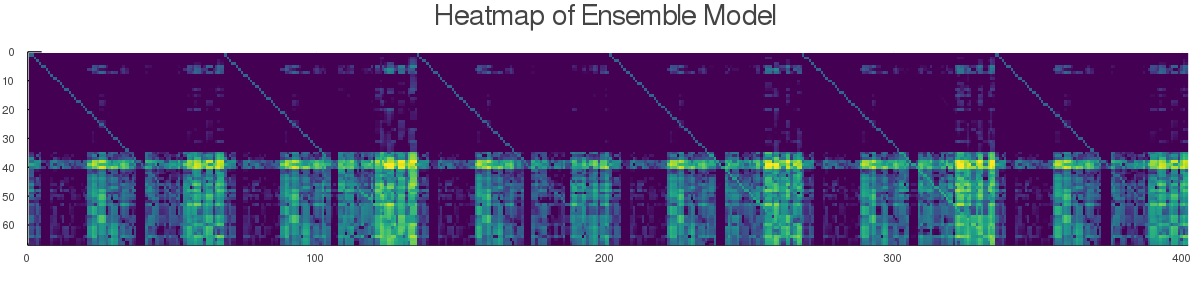}
  \end{center}
  \caption{Here we render a representative heatmap of the learned $W$ model from the 6 ensemble model features. Visible are 6 diagonal traces acting as a weighted average of the output of each member of the ensemble, but also significant off-diagonal values. The upper block values represent generalized positions, while the lower block is generalized velocities. Critical to modeling contact forces is accurate prediction of velocity.}
\end{figure}

\begin{table}
	\caption{Hyper-parameters used for Armhand environment.}
	\label{tab:paramarmhand}
	\centering
	\begin{tabular}{l|c||l|l}
		\toprule
		Hyper-parameter& Value & Ensemble Parameter & Value\\
		\midrule
		variance of controls             &$0.2^2$& Models in Ensemble & $6$ \\
		temperature parameter            &$0.08$  & Mass & $\mathcal{U}(0.01, 1.0)$ \\
		planning horizon                 &$50$   & Center of Mass & $\mathcal{U}(-0.04, 0.04)\times3$ \\
		number of planning samples       &$64$  & Moments of Inertia & $\mathcal{U}(0.0001, 0.0004)$ \\
		prior parameter                  &$0.0001$ & Contact Param. (solimp) & $[\mathcal{U}(0.5,0.99),$\\
		&&&$\mathcal{U}(0.4,0.98),$ \\
		&&&$\mathcal{U}(0.0001,0.01),$ \\
		&&&$\mathcal{U}(0.49,0.51),$ \\
		&&&$\mathcal{U}(1.9,2.1)]$ \\
		posterior reshaping constant         &$0.01$    & Contact Param. (solref) & $[\mathcal{U}(0.01,0.03),$\\&&&$\mathcal{U}(0.9,1.1)]$ \\
		episodes between model updates       &$1$    & & \\
		\bottomrule
	\end{tabular}
\end{table}

\begin{table}
	\caption{Learned model predictive error for different features.}
	\label{tab:prederrarmhand}
	\centering
	\begin{tabular}{l|c}
		\toprule
		Feature method & Predictive Error: \\
		& $\|x_{h+1} -W\phi\|_2/\|W\phi\|_2$ \\
		\midrule
		Random Fourier Features, 2048 & 0.22 \\
		2 Layer Neural Network, 2048 hidden, $relu$ activation & 0.41 \\
		Model Ensemble of 6 models & 0.09\\
		\bottomrule
	\end{tabular}
\end{table}

\end{document}